\documentclass{article}

\usepackage{xcolor}
\usepackage{latexcolors}
\usepackage{enumerate, amsmath, amsthm, amssymb, algorithm, algorithmic, pifont, fullpage, csquotes, dashrule, tikz, bbm, booktabs, bm, verbatim, url, mathrsfs}
\usepackage{dsfont}
\usepackage[framemethod=TikZ]{mdframed}
\usepackage{natbib}
\definecolor{stanfordred}{rgb}{0.54901961, 0.08235294, 0.08235294}
\usepackage[colorlinks,backref=page]{hyperref}
\hypersetup{
     colorlinks = true,
     linkcolor = blue,
     anchorcolor = blue,
     citecolor = stanfordred,
     filecolor = blue,
     urlcolor = blue
     }
\usepackage{caption}
\usepackage{subcaption}
\usepackage{authblk}
\usepackage{epigraph}
\definecolor{dblue}{RGB}{98, 140, 190}
\definecolor{dlblue}{RGB}{216, 235, 255}
\definecolor{dgreen}{RGB}{124, 155, 127}
\definecolor{dpink}{RGB}{207, 166, 208}
\definecolor{dyellow}{RGB}{255, 248, 199}
\definecolor{dgray}{RGB}{46, 49, 49}

\newcommand{\durl}[1]{\textcolor{dblue}{\underline{\url{#1}}}}

\newcommand{\ubr}[1]{\underbrace{#1}}

\newcommand{\eps}{\varepsilon}

\newcommand{\mc}[1]{\mathcal{#1}}

\newcommand{\bE}{\mathbb{E}}

\newcommand{\bR}{\mathbb{R}}
\newcommand{\bP}{\mathbb{P}}

\newcommand{\bI}{\mathbb{I}}
\newcommand{\bH}{\mathbb{H}}
\newcommand{\bN}{\mathbb{N}}

\newcommand{\kl}[2]{D_{\mathrm{KL}}(#1 \mid\mid #2)}

\newcommand{\expect}[1]{\mathbb{E}\left[#1\right]}

\newcommand{\ra}{\rightarrow}

\DeclareMathOperator*{\argmax}{arg\,max}

\newmdenv[
  topline=false,
  bottomline=false,
  rightline = false,
  leftmargin=10pt,
  rightmargin=0pt,
  innertopmargin=0pt,
  innerbottommargin=0pt
]{innerproof}

\newcounter{DaveDefCounter}
\newtheorem{corollary}{Corollary}

\newtheorem{lemma}{Lemma}

\newtheorem{theorem}{Theorem}
\newtheorem{fact}{Fact}

\newif\ifsubmit
\submitfalse
\ifsubmit
\newcommand{\dnote}[1]{}
\newcommand{\mnote}[1]{}
\newcommand{\ndg}[1]{}
\newcommand{\bnote}[1]{}
\else
\newcommand{\dnote}[1]{\textcolor{blue}{Dilip: #1}}
\newcommand{\mnote}[1]{\textcolor{violet}{Mark: #1}}
\newcommand{\ndg}[1]{\textcolor{green}{Noah: #1}}
\newcommand{\bnote}
[1]{\textcolor{orange}{Ben: #1}}
\fi

\title{Bayesian Reinforcement Learning with Limited Cognitive Load}

\author[1]{Dilip Arumugam$^\text{\textsterling}$\thanks{\texttt{dilip@cs.stanford.edu}}}
\author[2]{Mark K. Ho$^\text{\textsterling}$\thanks{\texttt{mkh260@nyu.edu}}}
\author[3,1]{Noah D. Goodman\thanks{\texttt{ngoodman@stanford.edu}}}
\author[4,5]{Benjamin Van Roy\thanks{\texttt{bvr@stanford.edu}}}
\affil[1]{Department of Computer Science, Stanford University}
\affil[2]{Center for Data Science, New York University}
\affil[3]{Department of Psychology, Stanford University}
\affil[4]{Department of Electrical Engineering, Stanford University}
\affil[5]{Department of Management Science \& Engineering, Stanford University}
\date{}

\begin{document}
\def\thefootnote{\textsterling}\footnotetext{Equal contribution}
\maketitle

\begin{abstract}
All biological and artificial agents must learn and make decisions given limits on their ability to process information. As such, a general theory of adaptive behavior should be able to account for the complex interactions between an agent's learning history, decisions, and capacity constraints. Recent work in computer science has begun to clarify the principles that shape these dynamics by bridging ideas from \emph{reinforcement learning}, \emph{Bayesian decision-making}, and \emph{rate-distortion theory}. This body of work provides an account of \emph{capacity-limited Bayesian reinforcement learning}, a unifying normative framework for modeling the effect of processing constraints on learning and action selection. Here, we provide an accessible review of recent algorithms and theoretical results in this setting, paying special attention to how these ideas can be applied to studying questions in the cognitive and behavioral sciences.


\end{abstract}

\section{Introduction}

Cognitive science aims to identify the principles and mechanisms that underlie adaptive behavior. An important part of this endeavor is the development of unifying, normative theories that specify ``design principles'' that guide or constrain how intelligent systems respond to their environment~\citep{marr1982vision,anderson1990adaptive,lewis2014computational,griffiths2015rational,gershman2015computational}. For example, accounts of learning, cognition, and decision-making often posit a function that an organism is optimizing---\textit{e.g.}, maximizing long-term reward or minimizing prediction error---and test plausible algorithms that achieve this---\textit{e.g.}, a particular learning rule or inference process. Historically, normative theories in cognitive science have been developed in tandem with new formal approaches in computer science and statistics. This partnership has been fruitful even given differences in scientific goals (\textit{e.g.}, engineering artificial intelligence versus \emph{reverse}-engineering biological intelligence). Normative theories play a key role in facilitating cross-talk between different disciplines by providing a shared set of mathematical, analytical, and conceptual tools for describing computational problems and how to solve them~\citep{ho2022cognitive}.

This paper is written in the spirit of such cross-disciplinary fertilization. Here, we review recent work in computer science~\citep{arumugam2021deciding,arumugam2022deciding} that develops a novel approach for unifying three distinct mathematical frameworks that will be familiar to many cognitive scientists (Figure~\ref{fig:architectures}). The first is \emph{Bayesian inference}, which has been used to study a variety of perceptual and higher-order cognitive processes such as categorization, causal reasoning, and social reasoning in terms of inference over probabilistic representations~\citep{yuille2006vision,baker2009action,tenenbaum2011grow,battaglia2013simulation,collins2013cognitive}. The second is \emph{reinforcement learning}~\citep{sutton1998introduction}, which has been used to model key phenomena in learning and decision-making including habitual versus goal-directed choice as well as trade-offs between exploring and exploiting~\citep{Daw2012,dayan2008reinforcement,radulescu2019holistic,wilson2014humans}. The third is \emph{rate-distortion theory}~\citep{shannon1959coding,berger1971rate}, a subfield of information theory~\citep{shannon1948mathematical,cover2012elements}, which in recent years has been used to model the influence of capacity-limitations in perceptual and choice processes~\citep{sims2016rate,lai2021policy,zenon2019information,zaslavsky2021rate}. All three of these formalisms have been used as normative frameworks in the sense discussed above: They provide general design principles (\textit{e.g.}, rational inference, reward-maximization, efficient coding) that explain the function of observed behavior and constrain the investigation of underlying mechanisms.

Although these formalisms have been applied to analyzing individual psychological processes, less work has used them to study learning, decision-making, and capacity limitations holistically. One reason is the lack of principled modeling tools that comprehensively integrate these multiple normative considerations. The framework of \emph{capacity-limited Bayesian reinforcement learning}, originally developed by~\cite{arumugam2021deciding,arumugam2022deciding} in the context of machine learning, directly addresses the question of how to combine these perspectives. Our goal is to review this work and present its key developments in a way that will be accessible to the broader research community and can pave the way for future cross-disciplinary investigations. 

We present the framework in two parts. First, we discuss a formalization of capacity-limited Bayesian \emph{decision-making} that introduces an \emph{information bottleneck} between an agent's beliefs about the world and its actions. This motivates a novel family of algorithms for identifying decision-rules that optimally trade off reward and information. Through a series of simple toy simulations, we analyze a specific algorithm: a variant of Thompson Sampling~\citep{thompson1933likelihood} that incorporates such an information bottleneck. Afterwards, we turn more fully to capacity-limited Bayesian \emph{reinforcement learning}, in which a decision-maker is continuously interacting with and adapting to their environment. We report both novel simulations and previously-established theoretical results in several learning settings, including multi-armed bandits as well as continual and episodic reinforcement learning. One feature of this framework is that it provides tools for analyzing how the interaction between capacity-limitations and learning dynamics can influence learning outcomes. In the discussion, we explore how such analyses and our framework can be applied to questions in cognitive science. We also discuss similarities and differences between capacity-limited Bayesian reinforcement learning and existing proposals, including information-theoretic bounded rationality~\citep{ortega2011information,gottwald2019bounded}, policy compression~\citep{lai2021policy}, and resource-rational models based on principles separate from information theory~\citep{lieder2014algorithm,callaway2022rational,ho2022people}.


\begin{figure}[t]
\centering
\includegraphics[width=\linewidth]{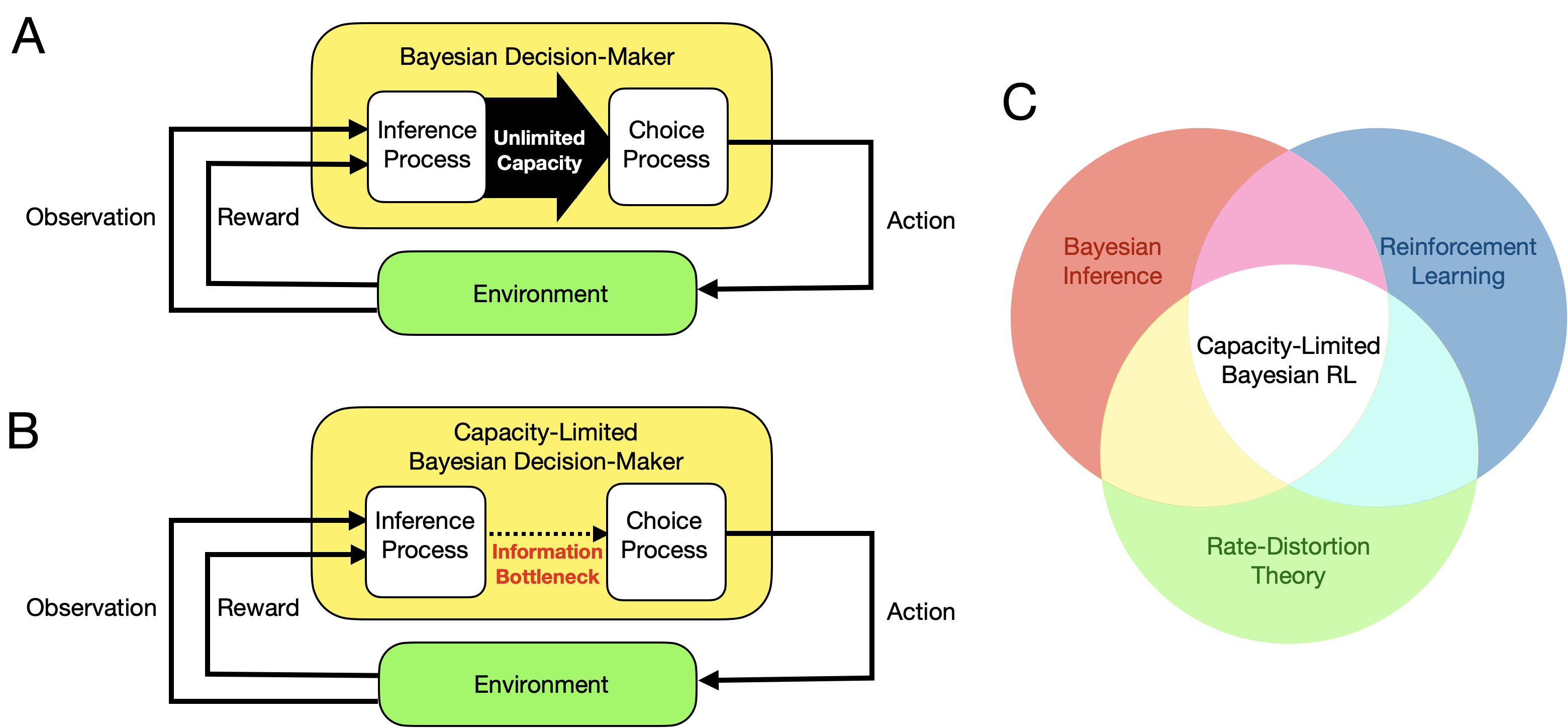}
\caption{(A) Bayesian learning and decision-making is typically modularized into distinct stages of \emph{inference} and \emph{choice}. That is, the decision-maker is conceptualized as mapping experiences to probabilistic beliefs about the environment (an inference process) and then performing computations based on the resulting beliefs to produce distributions over actions (a choice process). Inference and choice processes are usually specified independently and assume that the channel from one to the other has unlimited capacity (thick solid arrow). (B) In \emph{capacity-limited Bayesian decision-making}, there exists an information bottleneck between inferences and choices (narrow dotted arrow). Given the results of a fixed inference process (\textit{e.g.}, exact or approximate Bayesian inference), the optimal choice process trades off expected rewards and the mutual information (the \emph{rate}) between beliefs about the environment and the distribution over desirable actions. (C) Capacity-limited Bayesian reinforcement learning integrates ideas from \emph{Bayesian inference}~\citep{jaynes2003probability}, \emph{reinforcement learning}~\citep{kaelbling1996reinforcement}, and \emph{rate-distortion theory}~\citep{cover2012elements}.}
\label{fig:architectures}
\end{figure}

\section{Capacity-Limited Bayesian Decision-Making}

This section provides a review of Bayesian models before introducing a general account of \emph{capacity-limited Bayesian decision-making}. We then discuss and analyze a practical algorithm for computing capacity-limited Bayesian decision procedures based on Thompson Sampling.

\subsection{Bayesian Inference and Decision-Making}

Bayesian or probabilistic models have been used to characterize a range of psychological phenomena, including perception, categorization, feature learning, causal reasoning, social interaction, and motor control~\citep{kording2004bayesian,itti2009bayesian,ma2012organizing,goodman2016pragmatic}. One distinguishing feature of Bayesian models is that they separate learning and decision-making into two stages: \emph{inferring} a parameter of the environment and \emph{choosing} an action based on those inferences (Figure~\ref{fig:architectures}A).

Inference is formalized in terms of an \emph{environment-estimator}, a probability distribution over the unknown environment $\mc{E}$ that is updated based on the experiences of the agent. Formally, given a history of experiences $H_t$ up to time $t$, an environment-estimator $\eta_t$ is updated according to Bayes rule: 
\begin{equation}
\eta_t(\mc{E}) = \bP(\mc{E} \mid H_t) \propto \bP(H_t \mid \mc{E}) \bP(\mc{E}),
\label{eq:Bayesrule}
\end{equation}
where $\bP(H_t \mid \mc{E})$ is the likelihood and $\bP(\mc{E})$ is the prior probability assigned to $\mc{E}$.  Note that the environment-estimator $\eta_t$ takes the form of a probability mass function over environments.

Choice is formalized as a \emph{decision-rule}, which bases the selection of actions on the results of the inference process (\textit{e.g.}, Bayesian inference). Concretely, a decision-rule $\delta$ lies internal to the agent and is a probability mass function over actions given the identity of the environment $\mc{E}$. That is, if at timestep $t$, the agent samples a plausible environment $\theta \sim \eta_t$, then $\delta(A = a \mid \mc{E} = \theta)$ is the probability that any action $a \in \mc{A}$ is a desirable decision for the environment $\theta$. Given an environment-estimator $\eta_t$ and decision-rule $\delta$, we can then define the joint distribution
\begin{equation}
p_t(A, \mc{E}) \triangleq \bP(A, \mc{E} \mid H_t) = \delta(A \mid \mc{E})\eta_t(\mc{E}).
\label{eq:ea_channel}
\end{equation}
Finally, suppose we have a real-valued utility function $U(e, a)$ that defines the utility of an action $a$ for a particular version of the environment $e$ (later we discuss reinforcement learning and will consider specific utility functions that represent reward and/or value). Then the utility of an environment-estimator and decision-rule pair is the expected utility of the joint distribution they induce: $\mc{U}(\eta, \delta) = \mathbb{E}_{p_t(A, \mc{E})}[U(\mc{E}, A)]$.

This separation of inference and choice into an independent Bayesian estimator and decision-rule is commonly assumed throughout psychology, economics, and computer science~\citep{vonNeumannMorgenstern47,kaelbling1998planning,ma2019bayesian}. However, even if inference about the environment is exact, discerning what decisions are desirable from it incurs some non-trivial degree of cognitive load and the associated cost or limit on how much those inferences can inform choices remains unaccounted for. To remedy this, \cite{arumugam2021deciding,arumugam2022deciding} developed a framework for Bayesian learning and decision-making given an information bottleneck between inference and choice (Figure~\ref{fig:architectures}B). We now turn to how to extend the standard Bayesian framework to incorporate such capacity limitations.

\subsection{Choice with Capacity Limitations}

In \emph{capacity-limited Bayesian decision-making}, we make two modifications to the standard formulation. First, rather than pre-specifying a fixed decision-rule, we allow for the decision-rule $\delta_t$ to be chosen based on the current environment-estimator $\eta_t$; intuitively, this allows for a valuation of which decisions are desirable based on the agent's current knowledge of the world, $\eta_t$. Second, rather than allowing arbitrarily complex dependencies between environment estimates and actions, we can view the decision-rule $\delta_t$ as an \emph{estimate-to-action channel} that has limited capacity. We can formulate capacity limitations in a general way by considering the \emph{mutual information} or \emph{rate} of the estimate-to-action channel\footnote{
For a joint distribution $p(X, Y)$ the mutual information between random variables $X$ and $Y$ is:
$$
\mathbb{I}(X, Y) = \sum_{x, y} p(X = x, Y = y) \ln\left(\frac{p(X = x, Y = y)}{p(X = x)p(Y = y)}\right),
$$
where $p(X = \cdot)$ and $p(Y = \cdot)$ are the marginal distributions for $X$ and $Y$, respectively. Intuitively, the mutual information captures the degree to which two random variables are ``coupled''. For example, if one random variable is a bijective function of the other (i.e., there is a deterministic, one-to-one correspondence between realizations of $X$ and $Y$) then the mutual information will be a large positive number; conversely, if $X$ and $Y$ are completely independent of one another, then the mutual information is 0.
}~\citep{cover2012elements}. The notion of rate comes from rate-distortion theory, a sub-field of information theory that studies how to design efficient but lossy coding schemes~\citep{shannon1959coding,berger1971rate}. In particular, the rate of any channel quantifies the number of bits transmitted or communicated on average per data sample; in our context, this gives a precise mathematical form for how much decisions (channel outputs) are impacted by environment beliefs (channel inputs). Intuitively, the rate resulting from a decision rule captures the amount of \emph{coupling} between a decision-maker's estimates of the environment and actions taken. The central assumption of this framework is that greater estimate-to-action coupling is more cognitively costly.

Thus, formally, an optimal agent would use a decision-rule (estimate-to-action channel) that \emph{both maximizes utility and minimizes rate}. If we additionally assume that the environment-estimator $\eta_t$ is fixed and exact as Equation~\ref{eq:Bayesrule} (in Section \ref{sec:disc}, we consider relaxing this assumption), then the optimal capacity-limited decision-rule at time $t$ is given by:
\begin{equation}
\delta^\star_t = \argmax_{\delta_t} \bigg\{
\mc{U}(\eta_t, \delta_t) - \lambda C(\eta_t, \delta_t) 
\bigg\},
\label{eq:optimal_dr}
\end{equation}
where $\mc{U}(\eta_t, \delta_t)$ is the is the expected utility of the estimate-to-action channel induced by $\eta_t$ and $\delta_t$, the cost $C(\eta_t, \delta_t)$ is the rate of the channel, and $\lambda \geq 0$ is a parameter that trades off utility and rate.

Equation~\ref{eq:optimal_dr} defines an optimization target for a capacity-limited Bayesian decision-rule. However, having an optimization target does not tell us how difficult it is to find or approximate a solution, or what the solution is for a specific problem. In the next section, we discuss and analyze one illustrative procedure for finding $\delta^\star_t$ which then tethers the decision-rule to agent learning via Thompson Sampling~\citep{thompson1933likelihood,russo2018tutorial}. 

\subsection{Thompson Sampling with Capacity-Limitations}

Different decision-rules are distinguished by the type of representation they use and the algorithms that operate over those representations. For example, some decision-rules only use a \emph{point-estimate} of each action's expected reward, such as \emph{reward maximization}, $\varepsilon$-\emph{greedy reward maximization}~\citep{cesa1998finite,vermorel2005multi,kuleshov2014algorithms}, \emph{Boltzmann}/\emph{softmax} action selection~\citep{littman1996algorithms,kuleshov2014algorithms,asadi2017alternative}, or \emph{upper-confidence bound} (UCB) action selection~\citep{auer2002finite,auer2002using,kocsis2006bandit}. Some of these rules also provide parameterized levels of ``noisiness'' that facilitate random exploration---\textit{e.g.}, the probability of selecting an action at random in $\varepsilon$-greedy, the temperature in a Boltzmann distribution, and the bias factor in UCB.

\begin{algorithm}[t]
   \caption{Thompson Sampling}
   \label{alg:thompson}
\begin{algorithmic}
   \STATE {\bfseries Input:} Environment-estimator $\eta(E)$ Reward Function $R$, Action space $\mathcal{A}$
   \STATE {\bfseries Output:} Action $a' \in \mathcal{A}$
   \STATE $e \sim \eta(E)$ 
   \STATE $\mathcal{A}^\star = \{ a \in \mathcal{A} : R(e, a) = \max_b R(e, b)\}$
   \STATE $a' \sim \text{Uniform}(\mathcal{A}^\star)$
   \RETURN $a'$
\end{algorithmic}
\end{algorithm}

\begin{algorithm}[H]
   \caption{Blahut-Arimoto STS (BLASTS)~\citep{arumugam2021deciding}}
   \label{alg:value_blasts}
\begin{algorithmic}
   \STATE {\bfseries Input:} Environment-estimator $\eta(\mc{E})$, Rate parameter $\lambda \geq 0$, Blahut-Arimoto Iterations $K \in \mathbb{N}$, Utility Function $U$, Posterior sample count $Z \in \mathbb{N}$, Action space $\mathcal{A}$
   \STATE {\bfseries Output:} Action $a' \in \mathcal{A}$
   \STATE $e_1, ..., e_Z \sim \eta(\mc{E})$ 
   \STATE $\delta_0(a \mid e_z) = \frac{1}{|\mathcal{A}|}$, $\forall a \in \mathcal{A}, \forall z \in [Z]$
   \FOR{$k \in [K]$}
   \FOR{$a \in \mathcal{A}$}
   \STATE $q_k(a) = \frac{1}{Z}\sum_z \delta_k(a \mid e_z)$
   \STATE $\delta_{k+1}(a \mid e_z) \propto q_k(a) \exp\{\frac{1}{\lambda}U(e_z, a)\}$
   \ENDFOR
   \ENDFOR
   \STATE $z' \sim \text{Uniform}([Z])$
   \STATE $a' \sim \delta_{K}(\cdot \mid e_{z'})$
   \RETURN $a'$
\end{algorithmic}
\end{algorithm}



In the Bayesian setting, decision-rules can take advantage of \emph{distributional information} which captures epistemic uncertainty~\citep{der2009aleatory} reflected by the agent's knowledge, rather than the aleatoric uncertainty present due to random noise. For example, \emph{Thompson Sampling}~\citep{thompson1933likelihood,russo2018tutorial} makes explicit use of distributional information by first sampling an environment and then selecting the best action under the premise that the sampled version of the environment reflects reality. This specific mapping from the sampled candidate environment and the corresponding best action(s) constitutes a particular decision-rule. Selecting actions to execute within the environment by sampling from this decision rule constitutes a coherent procedure (formally outlined in Algorithm~\ref{alg:thompson}) that implicitly determines an action distribution given the current history of interaction $H_t$; sampling and executing actions in this manner is often characterized as \emph{probability matching}~\citep{agrawal2012analysis,agrawal2013further,russo2016information} where an action is only ever executed according to its probability of being optimal. Because Thompson Sampling is straightforward to implement and has good theoretical learning guarantees, it is frequently used in machine learning applications~\citep{chapelle2011empirical}. Additionally, humans often display key signatures of selecting actions via Thompson Sampling~\citep{vulkan2000economist,wozny2010probability,gershman2018deconstructing}. In short, Thompson Sampling is a simple, robust, and well-studied Bayesian algorithm that is, by design, tailored to a particular decision-rule. However, this decision-rule and, by extension, the standard version of Thompson Sampling as a whole, assumes that the estimate-to-action channel has unlimited capacity. What if, instead, we consider a version in which the rate is penalized and the decision-rule is optimized as in Equation~\ref{eq:optimal_dr}?

This consideration motivates Blahut-Arimoto Satisficing Thompson Sampling (BLASTS), an algorithm first proposed by~\cite{arumugam2021deciding}. In order to approximate an optimal decision-rule given an environment-estimator $\eta$ and rate parameter $\lambda \geq 0$, BLASTS (whose pseudocode appears as Algorithm \ref{alg:value_blasts}) performs three high-level procedures. First, it approximates the environment distribution by drawing $Z \in \bN$ Monte-Carlo samples from $\eta$ and proceeding with the resulting empirical distribution. Second, it uses Blahut-Arimoto---a classic algorithm from the rate-distortion theory literature~\citep{blahut1972computation,arimoto1972algorithm}---to iteratively compute the (globally) optimal decision-rule, $\delta^\star$, whose support is a finite action space $\mathcal{A}$. Finally, it uniformly samples one of the $Z$ initially drawn environment configurations $e'$ and then samples an action $a'$ from the computed decision-rule conditioned on that realization $e'$ of the environment. This last step allows for a generalized retention of the probability matching principle seen in  Thompson Sampling; that is, actions are only ever executed according to their probability of striking the right balance in Equation \ref{eq:optimal_dr}. One can observe that a BLASTS agent with no regard for respecting capacity limitations ($\lambda = 0$) will recover the Thompson Sampling decision-rule as a special case.

\begin{figure}[]
\centering
\includegraphics[width=.90\linewidth]{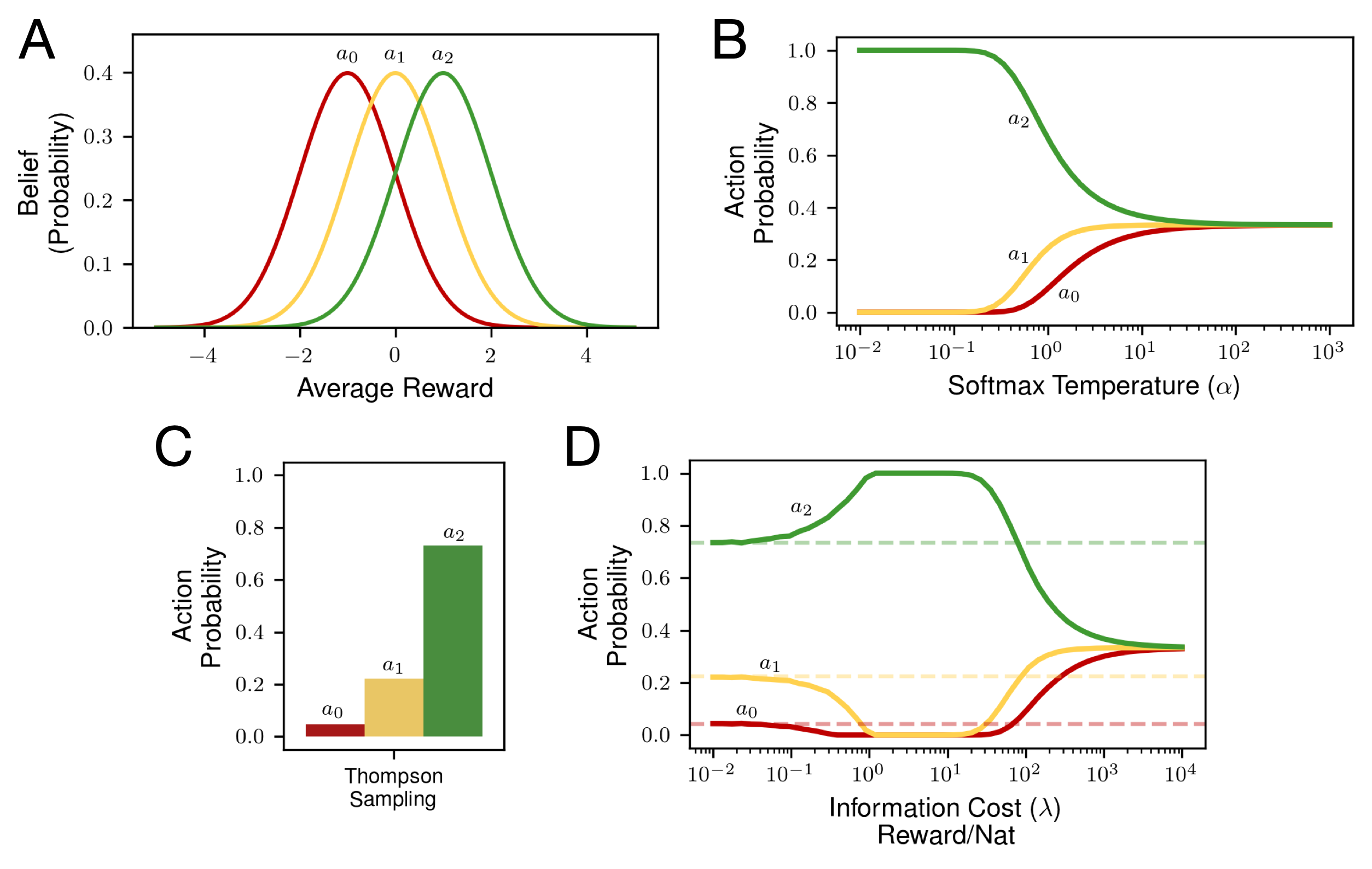}
\caption{
Capacity-limited decision-making in a three-armed bandit. (A) Bayesian decision-makers represent probabilistic uncertainty over their environment. Shown are Gaussian beliefs for average rewards for three actions, $a_0, a_1$, and $a_2$, with location parameters $\mu_0 = -1$, $\mu_1 = 0$, $\mu_2 = 1$, and standard deviations $\sigma_i = 1$ for $i = 0, 1, 2$. (B) A non-Bayesian decision-rule is the Boltzmann or soft-max distribution~\citep{littman1996algorithms}, which has a temperature parameter $\alpha > 0$. For the values in panel A, as $\alpha \rightarrow 0$, the action with the highest expected reward is chosen more deterministically; as $\alpha \rightarrow \infty$, actions are chosen uniformly at random. The Boltzmann decision-rule ignores distributional information. (C) An alternative decision-rule that is sensitive to distributional information is Thompson Sampling~\citep{thompson1933likelihood}, which implements a form of \emph{probability matching} that is useful for exploration~\citep{russo2016information}. Shown are the Thompson Sampling probabilities based on $N = 10,000$ samples. Thompson Sampling has no parameters. (D) In capacity-limited decision-making, action distributions that are more tightly coupled to beliefs about average rewards---i.e., those with higher mutual information or \emph{rate}---are penalized. The parameter $\lambda \geq 0$ controls the penalty and represents the cost of information in rewards per nat. Blahut-Arimoto Satisficing Thompson Sampling (BLASTS)~\citep{arumugam2021deciding} generalizes Thompson Sampling by finding the estimate-to-action channel that optimally trades off rewards and rate for a value of $\lambda$. In the current example, when $0 < \lambda \leq 10^{-1}$, information is cheap and BLASTS implements standard Thompson Sampling; when $10^{-1} \leq \lambda \leq 10^{1}$, BLASTS prioritizes information relevant to maximizing rewards and focuses on exploiting arms with higher expected reward, eventually only focusing on the single best; when $\lambda \geq 10^1$, information is too expensive to even exploit, so BLASTS resembles a Boltzmann distribution with increasing temperature, tending towards a uniform action distribution---that is, one that is completely uninformed by beliefs. Solid lines represent action probabilities according to BLASTS ($Z = 50,000$); dotted lines are standard Thompson Sampling probabilities for reference. 
}
\label{fig:simple_sims-1}
\end{figure}

Since BLASTS constructs the estimate-to-action channel that optimally trades off utility and rate, the action distribution it generates is primarily sensitive to the rate parameter, $\lambda$, and the environment-estimator, $\eta$. To illustrate the behavior of the optimal decision-rule, we conducted two sets of simulations that manipulated these factors in simple three-armed bandit tasks. Our first set of simulations examined the effect of different values of the rate parameter $\lambda$, which intuitively corresponds to the \emph{cost of information} measured in units of utils per nat. We calculated the marginal action distribution, $\pi(a) = \sum_{e}\delta^\star(a \mid e)\eta(e)$, where the belief distribution over average rewards for the three arms was represented by three independent Gaussian distributions respectively centered at $-1$, $0$, and $1$; all three distributions had a standard deviation of $1$ (Figure~\ref{fig:simple_sims-1}A). 

Remarkably, even on this simple problem, BLASTS displays three qualitatively different regimes of action selection when varying the rate parameter, $\lambda$, from $10^{-2}$ to $10^4$. When information is inexpensive ($\lambda < 10^{-1}$), the action distribution mimics the exploratory behavior of Thompson Sampling (consistent with theoretical predictions~\citep{arumugam2021deciding}). As information becomes moderately expensive ($10^{-1} \leq \lambda \leq 10^{1}$), BLASTS focuses channel capacity on the actions with higher expected utility by first reducing its selection of the worst action in expectation ($a_0$) followed by the second-worst/second-best action in expectation ($a_1$), which results in it purely exploiting the best action in expectation ($a_2$). Finally, as the util per nat becomes even greater ($\lambda \geq 10^1$) BLASTS produces actions that are \emph{uninformed} by its beliefs about the environment. This occurs in a manner that resembles a Boltzmann distribution with increasing temperature, eventually saturating at a uniform distribution over actions. These patterns are visualized in Figure~\ref{fig:simple_sims-1}B-D, which compare action probabilities for Boltzmann, Thompson Sampling, and BLASTS.

\begin{figure}[H]
\centering
\includegraphics[width=\linewidth]{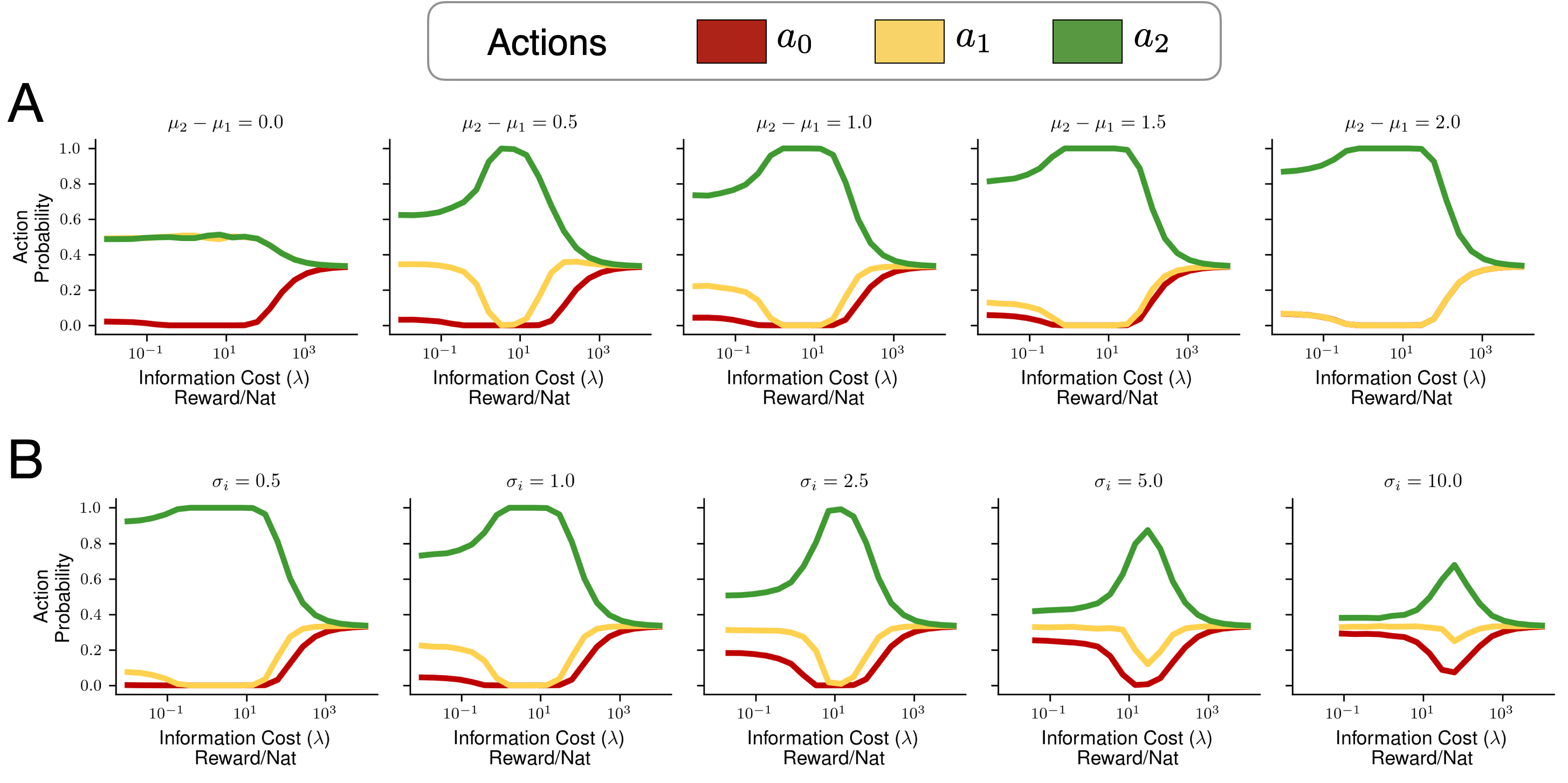}
\caption{
Blahut-Arimoto Satisficing Thompson Sampling (BLASTS) for different beliefs about average rewards in a three-armed bandit. (A) BLASTS is sensitive to the \emph{action gap}---the difference between the expected reward of the highest and second highest actions. Shown are action probability by information cost curves when $\mu_1$ from the example in Figure~\ref{fig:simple_sims-1}A is set to values in $\{-1.0, 0.5, 0.0, 0.5, 1.0\}$ and all other belief parameters are held constant. (B) BLASTS is also sensitive to the degree of uncertainty---\textit{e.g.}, the standard deviation of average reward estimates for each action. Shown are action probability / information cost curves when the standard deviation for each arm in Figure~\ref{fig:simple_sims-1}, $\sigma_i$, $i = 0, 1, 2$ is set to different values.
}
\label{fig:simple_sims-2}
\end{figure}

Our second set of simulations examine the relationship between the cost of information $\lambda$ and BLASTS action probabilities for different environment-estimates. Specifically, we first examined the effect of changing beliefs about the \emph{action gap}, the difference between the best and second-best action in expectation~\citep{auer2002finite,agrawal2012analysis,agrawal2013further,farahmand2011action,bellemare2016increasing}. As shown in Figure~\ref{fig:simple_sims-2}A, when the action gap is lower (corresponding to a more difficult decision-making task), BLASTS chooses the optimal action with lower probability for all values of $\lambda$. In addition, we examined the effect of changing uncertainty in the average rewards by setting different standard deviations for beliefs about the arms. Figure~\ref{fig:simple_sims-2}B shows that as uncertainty increases, BLASTS is less likely to differentially select an arm even in the ``exploitation'' regime for moderate values of $\lambda$. Sensitivity to the action gap and uncertainty are key features of BLASTS that derive from the fact that it uses distributional information to guide decision-making, unlike decision-rules such as $\varepsilon$-greedy or Boltzmann softmax.

\subsection{Summary}

In the standard formulation of Bayesian decision-making, we assume an environment-estimator and decision-rule that are specified independently. By extending ideas from rate-distortion theory, \cite{arumugam2021deciding} defined a notion of capacity-limitations applicable to decision-rules as well as an efficient algorithm for finding an optimal capacity-limited variant of Thompson Sampling (BLASTS). In this section, we analyzed how choice distributions change as a function of the cost of information and current environment estimates, which provides some intuition for how capacity-limitations affect choice from the agent's \emph{subjective} point of view. In the next section, we take a more \emph{objective} point of view by studying the learning dynamics that arise when capacity-limited agents interact with an environment over time.

\section{Capacity-Limited Bayesian Reinforcement Learning}



The preceding section provides a cursory overview of how rate-distortion theory accommodates capacity-limited learning within a Bayesian decision-making agent. In this section, we aim to provide mathematically-precise instantiations of the earlier concepts for three distinct problem classes: \textbf{(1)} continual or lifelong learning, \textbf{(2)} multi-armed bandits, and \textbf{(3)} episodic Markov decision processes. Of these three types of environments, the capacity-limited learning framework we provide for continual learning is a novel contribution of this work whereas the remaining two classes (which emerge as special cases of continual learning) have been examined in prior work~\citep{arumugam2021deciding,arumugam2021the,arumugam2022deciding}. 

\subsection{Preliminaries}
\label{sec:prelims}

In this section, we provide brief details on our notation and information-theoretic quantities used throughout the remainder of the paper. We encourage readers to consult \citep{cover2012elements,gray2011entropy,duchi21ItLectNotes,polyanskiy2022IT} for more background on information theory. We define all random variables with respect to a probability space $(\Omega, \mc{F}, \bP)$. For any two random variables $X$ and $Y$, we use the shorthand notation $p(X) \triangleq \bP(X \in \cdot)$ to denote the law or distribution of the random variable $X$ and, analogously, $p(X \mid Y) \triangleq \bP(X \in \cdot \mid Y)$ as well as $p(X \mid Y = y) \triangleq \bP(X \in \cdot \mid Y = y)$ for the associated conditional distributions given $Y$ and a realization of $Y$, respectively. For the ease of exposition, \textbf{we will assume throughout this work that all random variables are discrete}; aside from there being essentially no loss of generality by assuming this (see Equation 2.2.1 of \citep{duchi21ItLectNotes} or Theorem 4.5 of \citep{polyanskiy2022IT} for the Gelfand-Yaglom-Perez definition of divergence \citep{gelfand1959calculation,perez1959information}), extensions to arbitrary random variables taking values on abstract spaces are straightforward and any theoretical results presented follow through naturally to these settings. In the case of any mentioned real-valued or vector-valued random variables, one should think of these as discrete with support obtained from some suitably fine quantization such that the resulting discretization error is negligible.  For any natural number $N \in \bN$, we denote the index set as $[N] \triangleq \{1,2,\ldots,N\}$. For any arbitrary set $\mc{X}$, $\Delta(\mc{X})$ denotes the set of all probability distributions with support on $\mc{X}$. For any two arbitrary sets $\mc{X}$ and $\mc{Y}$, we denote the class of all functions mapping from $\mc{X}$ to $\mc{Y}$ as $\{\mc{X} \ra \mc{Y}\} \triangleq \{f \mid f:\mc{X} \ra \mc{Y}\}$.

We define the mutual information between any two random variables $X,Y$ through the Kullback-Leibler (KL) divergence: $$\bI(X;Y) = \kl{p(X,Y)}{p(X)p(Y)}, \qquad \kl{q_1}{q_2} = \sum\limits_{x \in \mc{X}} q_1(x)\log\left(\frac{q_1(x)}{q_2(x)}\right),$$ where $q_1, q_2 \in \Delta(\mc{X})$ are both probability distributions. An analogous definition of conditional mutual information holds through the expected KL-divergence for any three random variables $X,Y,Z$:
$$\bI(X;Y \mid Z) = \bE\left[\kl{p(X,Y \mid Z)}{p(X \mid Z)p(Y \mid Z)}\right].$$
With these definitions in hand, we may define the entropy and conditional entropy for any two random variables $X,Y$ as $$\bH(X) = \bI(X;X) \qquad \bH(Y \mid X) = \bH(Y) - \bI(X;Y).$$ This yields the following identities for mutual information and conditional mutual information for any three arbitrary random variables $X$, $Y$, and $Z$:
$$\bI(X;Y) = \bH(X) - \bH(X \mid Y) = \bH(Y) - \bH(Y | X), \qquad \bI(X;Y|Z) = \bH(X|Z) - \bH(X \mid Y,Z) = \bH(Y|Z) - \bH(Y | X,Z).$$
Finally, for any three random variables $X$, $Y$, and $Z$ which form the Markov chain $X \ra Y \ra Z$, we have the following data-processing inequality: $\bI(X;Z) \leq \bI(X;Y)$.

In subsequent sections, the random variable $H_t$ will often appear denoting the current history of an agent's interaction with the environment. We will use $p_t(X) = p(X \mid H_t)$ as shorthand notation for the conditional distribution of any random variable $X$ given a random realization of an agent's history $H_t$, at any timestep $t \in [T]$. Similarly, we denote the entropy and conditional entropy conditioned upon a specific realization of an agent's history $H_t$, for some timestep $t \in [T]$, as $\bH_t(X) \triangleq \bH(X \mid H_t = H_t)$ and $\bH_t(X \mid Y) \triangleq \bH_t(X \mid Y, H_t = H_t)$, for two arbitrary random variables $X$ and $Y$. This notation will also apply analogously to the mutual information $\bI_t(X;Y) \triangleq \bI(X;Y \mid H_t = H_t) = \bH_t(X) - \bH_t(X \mid Y) = \bH_t(Y) - \bH_t(Y \mid X),$ as well as the conditional mutual information $\bI_t(X;Y \mid Z) \triangleq \bI(X;Y \mid H_t = H_t, Z),$ given an arbitrary third random variable, $Z$. A reader should interpret this as recognizing that, while standard information-theoretic quantities average over all associated random variables, an agent attempting to quantify information for the purposes of exploration does so not by averaging over all possible histories that it could potentially experience, but rather by conditioning based on the particular random history $H_t$ that it has currently observed thus far. This dependence on the random realization of history $H_t$ makes all of the aforementioned quantities random variables themselves. The traditional notions of conditional entropy and conditional mutual information given the random variable $H_t$ arise by taking an expectation over histories: $$\begin{cases} \bE\left[\bH_t(X)\right] = \bH(X \mid H_t) \\ \bE\left[\bH_t(X \mid Y)\right] = \bH(X \mid Y, H_t)\end{cases}, \qquad \begin{cases}\bE\left[\bI_t(X;Y)\right] = \bI(X;Y \mid H_t), \\ \bE\left[\bI_t(X;Y \mid Z)\right] = \bI(X;Y \mid H_t,Z)\end{cases} .$$
Additionally, we will also adopt a similar notation to express a conditional expectation given the random history $H_t$: $\bE_t\left[X\right] \triangleq \bE\left[X \mid H_t\right].$

\subsection{Continual Learning}
\label{sec:continual}

At the most abstract level, we may think of a decision-making agent faced with a continual or lifelong learning setting~\citep{thrun1994finding,konidaris2006autonomous,wilson2007multi,lazaric2011transfer,brunskill2013sample,brunskill2015online,isele2016using,abel2018policy} within a single, stationary environment, which makes no further assumptions about Markovity or episodicity; such a problem formulation aligns with those of \citet{lu2021reinforcement,foster2021statistical,dong2022simple}, spanning multi-armed bandits and reinforcement-learning problems~\citep{lattimore2020bandit,sutton1998introduction}. More concretely, we adopt a generic agent-environment interface where, at each time period $t$, the agent executes an action $A_t \in \mc{A}$ within an environment $\mc{E} \in \Theta$ that results in an associated next observation $O_{t} \in \mc{O}$. This sequential interaction between agent and environment yields an associated history\footnote{At the very first timestep, the initial history only consists of an initial observation $H_0 = O_0 \in \mc{O}$.} at each timestep $t$, $H_t = (O_0, A_1, O_1, \ldots, A_{t-1}, O_{t-1}) \in \mc{H}$, representing the action-observation sequence available to the agent upon making its selection of its current action $A_t$. We may characterize the overall environment as $\mc{E} = \langle \mc{A}, \mc{O}, \rho \rangle \in \Theta$ containing the action set $\mc{A}$, observation set $\mc{O}$, and observation function $\rho: \mc{H} \times \mc{A} \ra \Delta(\mc{O})$, prescribing the distribution over next observations given the current history and action selection: $\rho(O_{t} \mid H_t, A_t) = \bP(O_{t} \mid \mc{E}, H_t, A_t)$. 

An agent's policy $\pi: \mc{H} \ra \Delta(\mc{A})$ encapsulates the relationship between the history encountered in each timestep $H_t$ and the executed action $A_t$ such that $\pi_t(a) = \bP(A_t = a \mid H_t)$ assigns a probability to each action $a \in \mc{A}$ given the history. Preferences across histories are expressed via a known reward function $r: \mc{H} \times \mc{A} \times \mc{O} \ra \bR$ so that an agent enjoys a reward $R_{t} = r(H_t, A_t, O_{t})$ on each timestep. Given any finite time horizon $T \in \bN$, the accumulation of rewards provide a notion of return $\sum\limits_{t=1}^{T} r(H_t, A_t, O_{t})$. To develop preferences over behaviors and to help facilitate action selection, it is often natural to associate with each policy $\pi$ a corresponding expected return or action-value function $Q^\pi: \mc{H} \times \mc{A} \ra \bR$ across the horizon $T$ as $Q^\pi(h,a) = \bE\left[\sum\limits_{t=1}^{T} r(H_t, A_t, O_{t}) \mid H_0 = h, A_0 = a , \mc{E} \right],$ where the expectation integrates over the randomness in the policy $\pi$ as well as the observation function $\rho$. Traditionally, an agent designer focuses on agents that strive to achieve the optimal value within the confines of some policy class $\Pi \subseteq \{\mc{H} \ra \Delta(\mc{A})\}$, $Q^\star(h,a) = \sup\limits_{\pi \in \Pi} Q^\pi(h,a)$, $\forall (h,a) \in \mc{H} \times \mc{A}$. The optimal policy then follows by acting greedily with respect to this optimal value function: $\pi^\star(h) = \argmax\limits_{a \in \mc{A}} Q^\star(h,a)$.


Observe that when rewards and the distribution of the next observation $O_{t}$ depend only on the current observation-action pair $(O_{t-1},A_t)$, rather than the full history $H_t$, we recover the traditional Markov Decision Process~\citep{bellman1957markovian,Puterman94} studied throughout the reinforcement-learning literature~\citep{sutton1998introduction}. Alternatively, when these quantities rely solely upon the most recent action $A_t$, we recover the traditional multi-armed bandit~\citep{lai1985asymptotically,bubeck2012regret,lattimore2020bandit}. Regardless of precisely which of these two problem settings one encounters, a default presumption throughout both literatures is that an agent should always act in pursuit of learning an optimal policy $\pi^\star$. Bayesian decision-making agents~\citep{bellman1959adaptive,duff2002optimal,ghavamzadeh2015bayesian} aim to achieve this by explicitly representing and maintaining the agent's current knowledge of the environment, recognizing that it is the uncertainty in the underlying environment $\mc{E}$ that drives uncertainty in optimal behavior $\pi^\star$. A Bayesian learner reflects this uncertainty through conditional probabilities $\eta_t(e) \triangleq \bP(\mc{E} = e \mid H_t)$, $\forall e \in \Theta$ aimed at estimating the underlying environment. Under the prior distribution $\eta_1(\mc{E})$, the entropy of this random variable $\mc{E}$ implies that a total of $\bH_1(\mc{E})$ bits quantify all of the information needed for identifying the environment and, as a result, synthesizing optimal behavior. For sufficiently rich and complex environments, however, $\bH_1(\mc{E})$ can become prohibitively large or even infinite, making the pursuit of an optimal policy entirely intractable. 

The core insight of this work is recognizing that a delicate balance between the amount of information processing that goes into a decision (\textit{cognitive load}) and the quality of that decision (\textit{utility}) can be aptly characterized through rate-distortion theory, providing a formal framework for capacity-limited decision making. At each time period $t \in [T]$, the agent's current knowledge about the underlying environment is fully specified by the distribution $\eta_t$. Whereas the standard Thompson Sampling (TS) agent will attempt to use this knowledge for identifying an optimal action $A^\star \in \argmax\limits_{a \in \mc{A}} Q^\star(H_t, a)$ by default, a capacity-limited agent may not be capable of operationalizing all bits of information from its beliefs about the world to discern a current action $A_t$. 

Rate-distortion theory~\citep{shannon1959coding,berger1971rate} is a branch of information theory~\citep{shannon1948mathematical,cover2012elements} dedicated to the study of lossy compression problems which necessarily must optimize for a balance between the raw amount of information retained in the compression and the utility of those bits for some downstream task; a classic example of this from the information-theory literature is a particular image that must be compressed down to a smaller resolution (fewer bits of information) without overly compromising the visual acuity of the content (bounded distortion). A capacity-limited agent will take its current knowledge $\eta_t$ as the information source to be compressed in each time period $t \in [T]$. The lossy compression mechanism or channel itself is simply a conditional probability distribution $p(A_t \mid \mc{E})$ that maps a potential realization of the unknown environment $\mc{E} \in \Theta$ to a corresponding distribution over actions for the current time period. Naturally, the amount of information used from the environment to identify this action is precisely quantified by the mutual information between these two random variables, $\bI_t(\mc{E}; A_t)$, where the $t$ subscript capture the dependence of the agent's beliefs $\eta_t$ on the current random history $H_t$. 

Aside from identifying the data to be compressed, a lossy compression problem also requires the specification of a distortion function $d: \mc{A} \times \Theta \ra \bR_{\geq 0}$ which helps distinguish between useful and irrelevant bits of 
information contained in the environment. Intuitively, environment-action pairs yielding high distortion are commensurate with achieving high loss and, naturally, a good lossy compression mechanism is one that can avoid large expected distortion, $\bE_t\left[d(A_t, \mc{E})\right].$ Putting these two pieces together, the fundamental limit of lossy compression is given by the rate-distortion function 
\begin{align}
    \mc{R}_t(D) &= \inf\limits_{p(A_t \mid \mc{E})} \bI_t(\mc{E}; A_t) \text{ such that } \bE_t\left[d(A_t, \mc{E})\right] \leq D,\label{eq:continual_rdf}
\end{align}
where we denote the conditional distribution that achieves this infimum as $\delta_t(\widetilde{A}_t \mid \mc{E})$ where $\widetilde{A}_t$ is the random variable representing this \textit{target action} that achieves the rate-distortion limit. A bounded decision maker with limited information processing can only hope to make near-optimal decisions. Thus, a natural way to quantify distortion is given by the expected performance shortfall between an optimal decision and the chosen one.  
$$d(a, \theta) = \bE_t\left[Q^\star(H_t, A^\star) - Q^\star(H_t, a)\mid \mc{E} = \theta\right].$$
The distortion threshold $D \in \bR_{\geq 0}$ input to the rate-distortion function is a free parameter specified by an agent designer that communicates a preferences for the minimization of rate versus the minimization of distortion. This aligns with a perspective that a decision-making agent has a certain degree of tolerance for sub-optimal behavior and, with that degree of error in mind, chooses among the viable near-optimal solutions that incur the least cognitive load to compute actions from beliefs about the world. If one is willing to tolerate significant errors and large amounts of regret, than decision-making should be far simpler in the sense that very few bits of information from beliefs about the environment are needed to select an action. Conversely, as prioritizing near-optimal behavior becomes more important, each decision requires greater cognitive effort as measure by the amount of information utilized to compute actions from current beliefs. The power of rate-distortion theory, in part, lies in the ability to give precise mathematical form to this intuitive narrative, as demonstrated by Fact \ref{fact:rdf}.
\begin{fact}[Lemma 10.4.1~\citep{cover2012elements}] For all $t \in [T]$ and any $D > 0$, the rate-distortion function $\mc{R}_t(D)$ is a non-negative, convex, and non-increasing function in its argument.
\label{fact:rdf}
\end{fact}
In particular, Fact \ref{fact:rdf} establishes the following relationship for any $D > 0$, $$\mc{R}_t(D) \leq
\mc{R}_t(0) \leq \bI_t(\mc{E};A^\star) = \bH_t(A^\star) - \ubr{\bH_t(A^\star \mid \mc{E})}_{\geq 0} \leq \bH_t(A^\star),$$ confirming that the amount of information used to determine $A_t$ is less than what would be needed to identify an optimal action $A^\star$. 

Alternatively, in lieu of presuming that an agent is cognizant of what constitutes a ``good enough'' solution, one may instead adopt the perspective that an agent is made aware of its capacity limitations. In this context, agent capacity refers to a bound $R \in \bR_{\geq 0}$ on the number of bits an agent may operationalize from its beliefs about the world in order to discern its current action selection $A_t$. Conveniently, the information-theoretic optimal solution is characterized by the Shannon distortion-rate function: 
\begin{align}
    \mc{D}_t(R) &= \inf\limits_{p(A_t \mid \mc{E})} \bE_t\left[d(A_t, \mc{E})\right] \text{ such that } \bI_t(\mc{E}; A_t) \leq R.\label{eq:continual_drf}
\end{align}
Natural limitations on a decision-making agent's time or computational resources can be translated and expressed as limitations on the sheer amount of information that can possibly be leveraged from beliefs about the environment $\mc{E}$ to execute actions; the distortion-rate function $\mc{D}_t(R)$ quantifies the fundamental limit on minimum expected distortion that an agent should expect under such a capacity constraint. It is oftentimes convenient that the rate-distortion function and distortion-rate function are inverses of one another such that $\mc{R}_t(\mc{D}_t(R)) = R$.

In this section, we have provided a mathematical formulation for how a capacity-limited agent might go about action selections in each time period that limit overall cognitive load in an information-theoretically optimal fashion while also leveraging as much of its environmental knowledge as possible to behave with limited sub-optimality. To elucidate the value of this formulation, we dedicate the following sections to simpler and more tractable problem settings which allow for theoretical and as well as empirical analysis.

\subsection{Multi-Armed Bandit}

In this section, we begin with the formal specification of a multi-armed bandit problem~\citep{lai1985asymptotically,bubeck2012regret,lattimore2020bandit} before presenting Thompson Sampling as a quintessential algorithm for identifying optimal actions. We then present a corresponding generalization of Thompson Sampling that takes an agent's capacity limitations into account.

\subsubsection{Problem Formulation}


We obtain a bandit environment as a special case of the problem formulation given in Section \ref{sec:continual} by treating the initial observation as null $O_0 = \emptyset$ while each subsequent observation denotes a reward signal $R_t \sim \rho(\cdot \mid A_t)$ drawn from an observation function $\rho: \mc{A} \ra \Delta(\bR)$ that only depends on the most recent action selection $A_t$ and not the current history $H_t = (A_1, R_1, A_2, R_2, \ldots, A_{t-1}, R_{t-1})$. While the actions $\mc{A}$ and total time periods $T \in \bN$ are known to the agent, the underlying reward function $\rho$ is unknown and, consequently, the environment $\mc{E}$ is itself a random variable such that $p(R_t \mid \mc{E}, A_t) = \rho(R_t \mid A_t).$ We let $\overline{\rho}: \mc{A} \ra [0,1]$ denote the mean reward function $\overline{\rho}(a) = \expect{R_t \mid A_t = a, \mc{E}}$, $\forall a \in \mc{A}$, and define an optimal action $A^\star \in \argmax\limits_{a \in \mc{A}} \overline{\rho}(a)$ as achieving the maximal mean reward denoted as $R^\star = \overline{\rho}(A^\star)$, both of which are random variables due to their dependence on $\mc{E}$.

Observe that, if the agent knew the underlying environment $\mc{E}$ exactly, there would be no uncertainty in the optimal action $A^\star$; consequently, it is the agent's epistemic uncertainty~\citep{der2009aleatory} in $\mc{E}$ that drives uncertainty in $A^\star$ and, since learning is the process of acquiring information, an agent explores to learn about the environment and reduce this uncertainty. As there is only a null history at the start $H_1 = \emptyset$, initial uncertainty in the environment $\mc{E} \in \Theta$ is given by the prior probabilities $\eta_1 \in \Delta(\Theta)$ while, as time unfolds, updated knowledge of the environment is reflected by posterior probabilities $\eta_t \in \Delta(\Theta)$.

For a fixed choice of environment $\mc{E}$, the performance of an agent is assessed through the regret of its policies over $T$ time periods $$\textsc{Regret}(\{\pi_t\}_{t \in [T]}, \mc{E}) = \expect{\sum\limits_{t=1}^T \left(\overline{\rho}(A^\star) - \overline{\rho}(A_t)\right) \mid \mc{E}}.$$
Since the environment is itself a random quantity, we integrate over this randomness with respect to the prior $\eta_1(\mc{E})$ to arrive at the Bayesian regret: $$\textsc{BayesRegret}(\{\pi_t\}_{t \in [T]}) = \expect{\textsc{Regret}(\{\pi_t\}_{t \in [T]}, \mc{E})} = \expect{\sum\limits_{t=1}^T \left(\overline{\rho}(A^\star) - \overline{\rho}(A_t)\right)}.$$
The customary goal within a multi-armed bandit problem is to identify an optimal action $A^\star$ and provably-efficient bandit learning emerges from algorithms whose Bayesian regret can be bounded from above. In the next section, we review one such algorithm that is widely used in practice before motivating consideration of satisficing solutions for bandit problems.

\subsubsection{Thompson Sampling \& Satisficing}

A standard choice of algorithm for addressing multi-armed bandit problems is Thompson Sampling (TS)~\citep{thompson1933likelihood,russo2018tutorial}, which has been well-studied both theoretically~\citep{auer2002finite,agrawal2012analysis,agrawal2013further,bubeck2013prior,russo2016information} and empirically~\citep{granmo2010solving,scott2010modern,chapelle2011empirical,gopalan2014thompson}. For convenience, we provide generic pseudocode for TS as Algorithm \ref{alg:ts}, whereas more granular classes of bandit problems (Bernoulli bandits or Gaussian bandits, for example) can often lead to more computationally explicit versions of TS that leverage special structure like conjugate priors (see \citep{russo2018tutorial} for more detailed implementations). In each time period $t \in [T]$, a TS agent proceeds by drawing one sample $\theta_t \sim \eta_t(\mc{E})$, representing a statistically-plausible hypothesis about the underlying environment based on the agent's current posterior beliefs from observing the history $H_t$; the agent then proceeds as if this sample dictates reality and acts optimally with respect to it, drawing an action to execute this time period $A_t$ uniformly at random among the optimal actions for this realization of $\mc{E} = \theta_t$ of the environment. Executing actions in this manner recovers the hallmark probability-matching principle~\citep{scott2010modern,russo2016information} of TS whereby, in each time period $t \in [T]$, the agent selects actions according to their (posterior) probability of being optimal given everything observed up to this point in $H_t$ or, more formally, $\pi_t(a) = p_t(A^\star = a)$, $\forall a \in \mc{A}$.

Aside from admitting a simple, computationally-efficient procedure for learning optimal actions $A^\star$ over time, TS also boasts rigorous theoretical guarantees. While the classic Gittins' indices~\citep{gittins1979bandit,gittins2011multi} yield the Bayes-optimal policy, they are extremely limited to problems of modest size such that, for our finite-horizon setting, they are computationally intractable. Nevertheless, \citet{russo2016information} offer a rigorous corroborating analysis of TS that, for our setting, yields an information-theoretic Bayesian regret bound: $$\textsc{BayesRegret}(\{\pi^{\text{TS}}_t\}_{t \in [T]}) \leq \sqrt{\frac{1}{2}|\mc{A}|\bH_1(A^\star)T} \leq \sqrt{\frac{1}{2}|\mc{A}|\log\left(|\mc{A}|\right)T}.$$ This result communicates that the overall Bayesian regret of TS is governed by the entropy over the optimal arm $A^\star$ under the prior $\eta_1(\mc{E})$. When an agent designer has strong prior knowledge about the optimal arm, initializing TS accordingly results in a very small upper bound on Bayesian regret; conversely, in the case of an uninformative prior, the worst-case entropy over the optimal arm is equal to $\log(|\mc{A}|)$ and the second inequality is tight, which still matches the best-known regret lower bound $\Omega(\sqrt{|\mc{A}|T})$ for multi-armed bandit problems up to logarithmic factors~\citep{bubeck2013prior}.

Naturally, a core premise of this work is to consider decision-making problems where an agent's inherent and unavoidable capacity limitations drastically impact the tractability of learning optimal actions. While there are other classes of algorithms for handling multi-armed bandit problems~\citep{auer2002finite,ryzhov2012knowledge,powell2012optimal,russo2014learning,russo2018learning}, TS serves an exemplary representative among them as it relentlessly pursues the optimal action $A^\star$, by design. Consider a human decision maker faced with a bandit problem containing $1,000,000,000$ (one trillion) arms -- does one genuinely expect any individual to successfully identify $A^\star$? Similarly, the final inequality in the Bayesian regret bound above informs us that the performance shortfall of TS will increase as the number of actions tends to $\infty$, quantifying the folly of pursuing $A^\star$ as the agent continuously experiments with untested but potentially optimal actions. 

\begin{center}
\begin{minipage}{0.48\textwidth}

\begin{algorithm}[H]
   \caption{Thompson Sampling (TS)~\citep{thompson1933likelihood}}
   \label{alg:ts}
\begin{algorithmic}
   \STATE {\bfseries Input:} Prior $p_1(\mc{E})$
   \FOR{$t \in [T]$}
   \STATE Sample $\theta_t \sim \eta_t(\mc{E})$
   \STATE $d(a,\theta_t) = \bE_t\left[\overline{\rho}(A_\star) - \overline{\rho}(a) \mid \mc{E} = \theta_t\right]$, $\forall a \in \mc{A}$
   \STATE $\pi_t = \text{Uniform}(\{a \in \mc{A} \mid d(a,\theta_t) = 0\})$
   \STATE Sample action $A_t \sim \pi_t$
   \STATE Observe reward $R_t$
   \STATE Update history $H_{t+1} = H_t \cup (A_t,R_t)$
   \ENDFOR
\end{algorithmic}
\end{algorithm}
\end{minipage}
\hfill
\begin{minipage}{0.48\textwidth}
\begin{algorithm}[H]
   \caption{Satisficing TS~\citep{russo2022satisficing}}
   \label{alg:sts}
\begin{algorithmic}
   \STATE {\bfseries Input:} Prior $p_1(\mc{E})$, Threshold $\eps \geq 0$
   \FOR{$t \in [T]$}
   \STATE Sample $\theta_t \sim \eta_t(\mc{E})$
   \STATE $d(a,\theta_t) = \bE_t\left[\overline{\rho}(A_\star) - \overline{\rho}(a) \mid \mc{E} = \theta_t\right]$, $\forall a \in \mc{A}$
   \STATE $\pi_t = \min(\{a \in \mc{A} \mid d(a,\theta_t) \leq \eps\})$
   \STATE Sample action $A_t \sim \pi_t$
   \STATE Observe reward $R_t$
   \STATE Update history $H_{t+1} = H_t \cup (A_t,R_t)$
   \ENDFOR
\end{algorithmic}
\end{algorithm}
\end{minipage}
\end{center}

Satisficing is a longstanding, well-studied idea about how to understand resource-limited cognition~\citep{simon1955behavioral,simon1956rational,newell1958elements,newell1972human,simon1982models} in which an agent settles for the first recovered solution that is deemed to be ``good enough,'' for some suitable notion of goodness. Inspired by this idea, \citet{russo2018satisficing,russo2022satisficing} present the Satisficing Thompson Sampling (STS) algorithm, which we present as Algorithm \ref{alg:sts}, to address the shortcomings of algorithms like TS that relentlessly pursue $A^\star$. STS employs a minimal adjustment to the original TS algorithm through a threshold parameter $\eps \geq 0$, which an agent designer may use to communicate that identifying a $\eps$-optimal action would be sufficient for their needs. The use of a minimum over all such $\eps$-optimal actions instead of a uniform distribution reflects the idea of settling for the first solution deemed to be ``good enough'' according to $\eps$. Naturally, the intuition follows that as $\eps$ increases and the STS agent becomes more permissive, such $\eps$-optimal actions can be found in potentially far fewer time periods than what is needed to obtain $A^\star$ through TS. If we define an analogous random variable to $A^\star$ as $A_\eps \sim \text{Uniform}(\{a \in \mc{A} \mid \bE_t\left[\overline{\rho}(A^\star) - \overline{\rho}(a) \mid \mc{E} = \theta_t\right] \leq \eps\})$ then STS simply employs probability matching as $\pi_t(a) = p_t(A_\eps = a)$, $\forall a \in \mc{A}$ and, as $\eps \downarrow 0$, recovers TS as a special case. \citet{russo2022satisficing} go on to prove a complementary information-theoretic regret bound for STS, which depends on $\bI_1(\mc{E}; A_\eps)$, rather than the entropy of $A^\star$, $\bH_1(A^\star)$. 

While it is clear that STS does embody the principle of satisficing for a capacity-limited decision maker, the $A_\eps$ action targeted by a STS agent instead of $A^\star$ only achieves some arbitrary and unspecified trade-off between the simplicity of what the agent set out to learn and the utility of the resulting solution, as $\eps$ varies. This is in contrast to a resource-rational approach~\citep{anderson1990adaptive,griffiths2015rational} which aims to instead strike the best trade-off between these two competing interests. One interpretation of the next section is that we provide a mathematically-precise characterization of such resource-rational solutions through rate-distortion theory.

\subsubsection{Rate-Distortion Theory for Target Actions}

To see how the rate-distortion function (Equation \ref{eq:continual_rdf}) fits into the preceding discussion of Thompson Sampling and STS, \citet{arumugam2021deciding} replace the $A_t$ of Equation \ref{eq:continual_rdf} with a target action $\widetilde{A}_t$. This notion of a target action based on the observation is that $A^\star = f(\mc{E})$ is merely a statistic of the environment whose computation is determined by some abstract function $f$. It follows that an alternative surrogate action an agent may prioritize during learning will be some other computable statistic of the environment that embodies a kind of trade-off between two key properties: (1) ease of learnability and (2) bounded sub-optimality or performance shortfall relative to $A^\star$. 

The previous section already gives two concrete examples of potential target actions, 
$A^\star$ and $A_\eps$, where the former represents an extreme point on the spectrum of potential learning targets as one that demands a potentially intractable amount of information to identify but comes with no sub-optimality. At the other end of the spectrum, there is simply the uniform random action $\overline{A} \sim \text{Uniform}(\mc{A})$ which requires no learning or sampling on the part of the agent to learn it but, in general, will likely lead to considerably large performance shortfall relative to an optimal solution. While, for any fixed $\eps > 0$, $A_\eps$ lives in between these extremes, it also suffers from two shortcomings of its own. Firstly, by virtue of satisficing and a willingness to settle for anything that is ``good enough,'' it is unclear how well $A_\eps$ balances between the two aforementioned desiderata. In particular, the parameterization of $A_\eps$ around $\eps$ as an upper bound to the expected regret suggests that there could exist an even simpler target action than $A_\eps$ that is also $\eps$-optimal but easier to learn insofar as it requires the agent obtain fewer bits of information from the environment. Secondly, from a computational perspective, a STS agent striving to learn $A_\eps$ (just as a TS agent does for learning $A^\star$) computes the same statistic repeatedly across all $T$ time periods. Meanwhile, with every step of interaction, the agent's knowledge of the environment $\mc{E}$ is further refined, potentially changing the outlook on what can be tractably learned in subsequent time periods. This suggests that one stands to have considerable gains by designing agents that adapt their learning target as knowledge of the environment accumulates, rather than iterating on the same static computation.

Recall that from Equation \ref{eq:continual_rdf}, a target action $\widetilde{A}_t$ following distribution $\delta_t(\widetilde{A}_t \mid \mc{E})$ achieves the rate-distortion limit given by
\begin{align}
    \mc{R}_t(D) &= \inf\limits_{p(\widetilde{A} \mid \mc{E})} \bI_t(\mc{E}; \widetilde{A}) \text{ such that } \bE_t\left[d(\widetilde{A}, \mc{E})\right] \leq D.\label{eq:target_action_rdf}
\end{align}

In order to satisfy the second desideratum of bounded performance shortfall for learning targets and to facilitate a regret analysis, \citet{arumugam2021deciding} define the distortion function as $$d(\tilde{a}, \theta) = \bE_t\left[\left(\overline{\rho}(A^\star) - \overline{\rho}(\tilde{a})\right)^2 \mid \mc{E} = \theta\right].$$ While having bounded expected distortion satisfies our second criterion for a learning target, the fact that $\widetilde{A}_t$ requires fewer bits of information to learn is immediately given by properties of the rate-distortion function $\mc{R}_t(D)$ itself, through Fact \ref{fact:rdf}.



\begin{center}
    \begin{minipage}{0.75\textwidth}
    \begin{algorithm}[H]
   \caption{Rate-Distortion Thompson Sampling (RDTS)}
   \label{alg:blasts}
\begin{algorithmic}
   \STATE {\bfseries Input:} Prior $\eta_1(\mc{E})$, Distortion threshold $D \geq 0$
   \FOR{$t \in [T]$}
   \STATE Compute $\delta_t(\widetilde{A}_t \mid \mc{E})$ that achieves $\mc{R}_t(D)$ limit (Equation \ref{eq:target_action_rdf})
   \STATE Sample $\theta_t \sim p_t(\mc{E})$
   \STATE Sample action $A_t \sim \delta_t(\widetilde{A}_t \mid \mc{E} = \theta_t)$
   \STATE Observe reward $R_t$
   \STATE Update history $H_{t+1} = H_t \cup (A_t,R_t)$
   \ENDFOR
\end{algorithmic}
\end{algorithm}
    \end{minipage}
\end{center}

Abstractly, one could consider a procedure like Algorithm \ref{alg:blasts} that, for an input distortion threshold $D$, identifies the corresponding target action $\widetilde{A}_t$ of Equation \ref{eq:target_action_rdf} and then performs probability matching with respect to it. The following theorem provides an information-theoretic Bayesian regret bound that generalizes the performance guarantee of traditional TS by \citet{russo2016information} while also providing a more direct connection to the rate-distortion function than Theorem 3 of \citet{arumugam2021deciding} using proof techniques developed by \citet{arumugam2022deciding}. 

\begin{theorem}
For any $D \geq 0$, $$\textsc{BayesRegret}(\{\pi^{\mathrm{RDTS}}_t\}_{t \in [T]}) \leq \sqrt{\frac{1}{2} |\mc{A}|T\mc{R}_1(D)} + T\sqrt{D}.$$
\label{thm:bandit_rdf_regret_bound}
\end{theorem}

When $D = 0$ and the agent designer is not willing to tolerate any sub-optimality relative to $A^\star$, Fact \ref{fact:rdf} allows this bound to recover the guarantee of TS exactly. At the other extreme, increasing $D$ to 1 (recall that mean reward are bounded in $[0,1]$) allows $\mc{R}_1(D) = 0$ and the agent has nothing to learn from the environment but also suffers the linear regret of $T$. Naturally, the ``sweet spot'' is to entertain intermediate values of $D$ where smaller values will lead to larger amounts of information $\mc{R}_1(D)$ needed to identify the corresponding target action, but not as many bits as what learning $A^\star$ necessarily entails.

Just as in the previous subsection, it may often be sensible to also consider a scenario where an agent designer is unable to precisely specify a reasonable threshold on expected distortion $D$ and can, instead, only characterize a limit on the amount of information an agent may acquire from the environment $R > 0$. One might interpret this as a notion of capacity which differs quite fundamentally from other notions examined in prior work~\citep{lai2021policy,gershman2021rational} (see Section \ref{sec:disc} for a more in-depth comparison). For this, we may consider the distortion-rate function 
\begin{align}
    \mc{D}_t(R) = \inf\limits_{p(\widetilde{A} \mid \mc{E})} \bE_t\left[d(\widetilde{A},\mc{E})\right] \text{ such that } \bI_t(\mc{E}; \widetilde{A}) \leq R,
    \label{eq:target_action_drf}
\end{align}
which quantifies the fundamental limit of lossy compression subject to a rate constraint, rather than the distortion threshold of $\mc{R}(D)$. Similar to the rate-distortion function, however, the distortion rate function also adheres to the three properties outlined in Fact \ref{fact:rdf}. More importantly, it is the inverse of the rate-distortion function such that $\mc{R}_t(\mc{D}_t(R)) = R$ for any $t \in [T]$ and $R > 0$. Consequently, by selecting $D = \mc{D}_1(R)$ as input to Algorithm \ref{alg:blasts}, we immediately recover the following corollary to Theorem \ref{thm:bandit_rdf_regret_bound} that provides an information-theoretic Bayesian regret bound in terms of agent capacity, rather than a threshold on expected distortion.

\begin{corollary}
For any $R > 0$, $$\textsc{BayesRegret}(\{\pi^{\mathrm{RDTS}}_t\}_{t \in [T]}) \leq \sqrt{\frac{1}{2} |\mc{A}|TR} + T\sqrt{\mc{D}_1(R)}.$$
\label{thm:bandit_drf_regret_bound}
\end{corollary}
The semantics of this performance guarantee are identical to those of Theorem \ref{thm:bandit_rdf_regret_bound}, only now expressed explicitly through the agent's capacity $R$. Namely, when the agent has no capacity for learning $R = 0$, $D_1(R) = 1$ and the agent incurs linear regret of $T$. Conversely, with sufficient capacity $R = \bH_1(A^\star)$, $D_1(R) = 0$ and we recover the regret bound of Thompson Sampling. Intermediate values of agent capacity will result in an agent that fully utilizes its capacity to acquire no more than $R$ bits of information from the environment, resulting in the minimum possible expected distortion quantified by $\mc{D}_1(R)$.

While a non-technical reader of this section should remain unencumbered by the mathematical minutia of these theoretical results, the salient takeaway is an affirmation that rate-distortion theory not only provides an intuitive and mathematically-precise articulation of capacity-limited Bayesian decision-making in multi-armed bandits, but also facilitates the design of a complementary algorithm for statistically-efficient learning. The next section proceeds to illustrate how these theoretical results hold up in practice.  


\subsubsection{Experiments}

In order to make the algorithm of the previous section (Algorithm \ref{alg:blasts}) amenable to practical implementation, \citet{arumugam2021deciding} look to the classic Blahut-Arimoto algorithm~\citep{blahut1972computation,arimoto1972algorithm}. Just as TS and STS perform probability matching with respect to $A^\star$ and $A_\eps$ in each time period, respectively, the Blahut-Arimoto STS (BLASTS) algorithm (presented as Algorithm \ref{alg:value_blasts} where one should recall that reward maximization and regret minimization are equivalent) conducts probability matching with respect to $\widetilde{A}_t$ in each time period to determine the policy: $\pi_t(a) = p_t(\widetilde{A}_t = a)$, $\forall a \in \mc{A}$. For two discrete random variables representing an uncompressed information source and the resulting lossy compression, the Blahut-Arimoto algorithm computes the channel that achieves the rate-distortion limit (that is, achieve the infimum in Equation \ref{eq:target_action_rdf}) by iterating alternating update equations until convergence. More concretely, the algorithm is derived by optimizing the Lagrangian of the constrained optimization~\citep{boyd2004convex} that is the rate-distortion function, which is itself known to be a convex optimization problem~\citep{chiang2004geometric}. We refer readers to \citep{arumugam2021deciding} for precise computational details of the Blahut-Arimoto algorithm for solving the rate-distortion function $\mc{R}_t(D)$ that yields $\widetilde{A}_t$ as well as \citep{arumugam2021the} for details on the exact theoretical derivation. 

One salient detail that emerges from using the Blahut-Arimoto algorithm in this manner is that it an agent designer's no longer specifies a distortion threshold $D \in \bR_{\geq 0}$ as input but, instead, provides a value of the Lagrange multiplier $\beta \in \bR_{\geq 0}$; lower values of $\beta$ communicate a preferences for rate minimization whereas larger values of $\beta$ prioritize distortion minimization. To each value of $\beta$, there is an associate distortion threshold $D$ as $\beta$ represents the desired slope achieved along the corresponding rate-distortion curve~\citep{blahut1972computation,csiszar1974computation,csiszar1974extremum}. As, in practice, $\eta_t(\mc{E})$ tends to be a continuous distribution, \citet{arumugam2021deciding} induce a discrete information source by drawing a sufficiently large number of Monte-Carlo samples and leveraging the resulting empirical distribution, which turns out to be a theoretically sound estimator of the true rate-distortion function~\citep{harrison2008estimation,palaiyanur2008uniform}. 

As these target actions $\{\widetilde{A}_t\}_{t \in [T]}$ are born out of a need to balance the simplicity and utility of what an agent aims to learn from its interactions within the environment, we can decompose empirical results into those that affirm these two criteria are satisfied in isolation. Since assessing utility or, equivalently, performance shortfall is a standard evaluation metric used throughout the literature, we begin there and offer regret curves in Figure \ref{fig:bandit_regret} for Bernoulli and Gaussian bandits with 10 independent arms (matching, for example, the empirical evaluation of \citet{russo2018learning}); recall that the former implies Bernoulli rewards $R_t \sim \text{Bernoulli}(\overline{\rho}(A_t))$ while the latter yields Gaussian rewards with unit variance $R_t \sim \mc{N}(\overline{\rho}(A_t), 1)$. We evaluate TS and BLASTS agents where, for the latter, the Lagrange multiplier hyperparameter $\beta \in \bR_{\geq 0}$ is fixed and tested over a broad range of values. All agents begin with a $\text{Beta}(1,1)$ prior for each action of the Bernoulli bandit and a $\mc{N}(0,1)$ prior for the Gaussian bandit. For each individual agent, the cumulative regret incurred by the agent is plotted over each time period $t \in [T]$. 

Recalling that our distortion function is directly connected to the expected regret of the BLASTS agent, we observe that smaller values of $\beta$ so aggressively prioritize rate minimization that the resulting agents incur linear regret; in both bandit environments, this trend persists for all values $\beta \leq 100$. Notably, as $\beta \uparrow \infty$, we observe the resulting agents yield performance more similar to regular TS. This observation aligns with expectations since, for a sufficiently large value of $\beta$, the Blahut-Arimoto algorithm will proceed to return a channel that only places probability mass on the distortion-minimizing actions, which are indeed, the optimal actions $A^\star$ for each realization of the environment. A notable auxiliary finding in these results, also seen in the original experiments of \citet{arumugam2021deciding}, is that intermediate values of $\beta$ manage to yield regret curves converging towards the optimal policy more efficiently that TS; this is, of course, only possible when the distortion threshold $D$ implied by a particular setting of $\beta$ falls below the smallest action gap of the bandit problem.

\begin{figure}[H]
\centering
\begin{subfigure}{.5\textwidth}
  \centering
  \includegraphics[width=.85\linewidth]{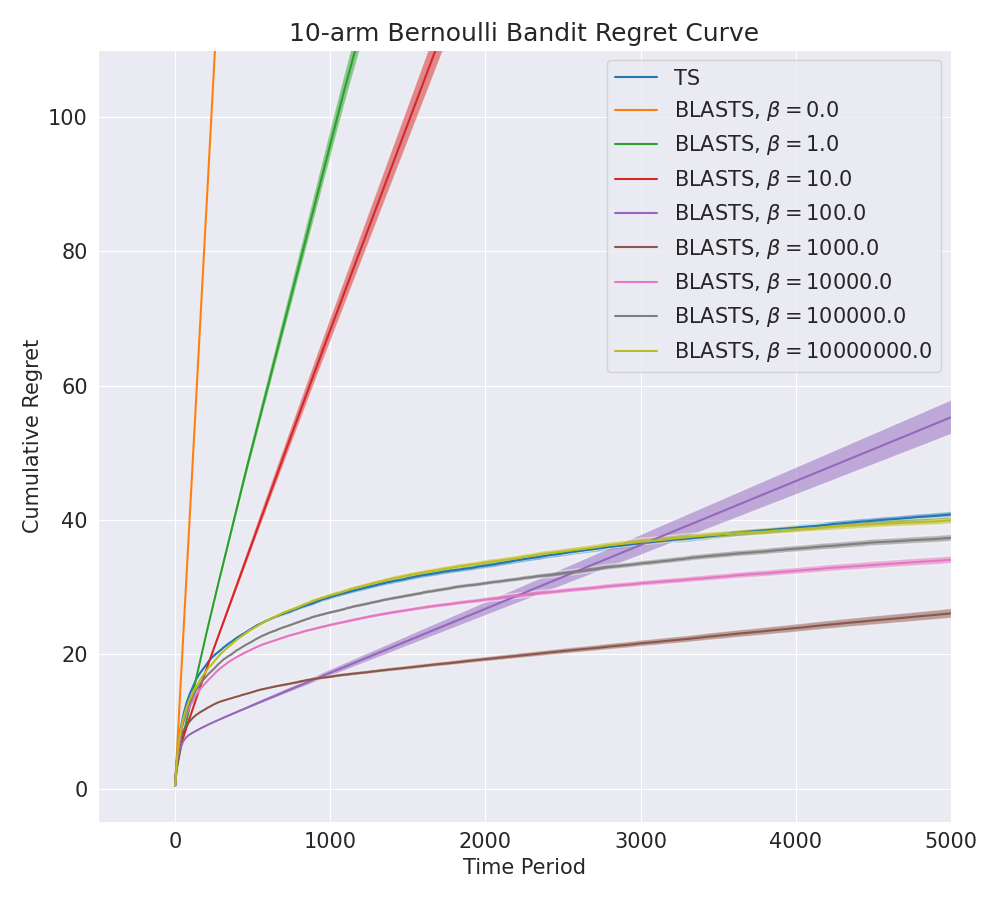}
\end{subfigure}%
\begin{subfigure}{.5\textwidth}
  \centering
  \includegraphics[width=.85\linewidth]{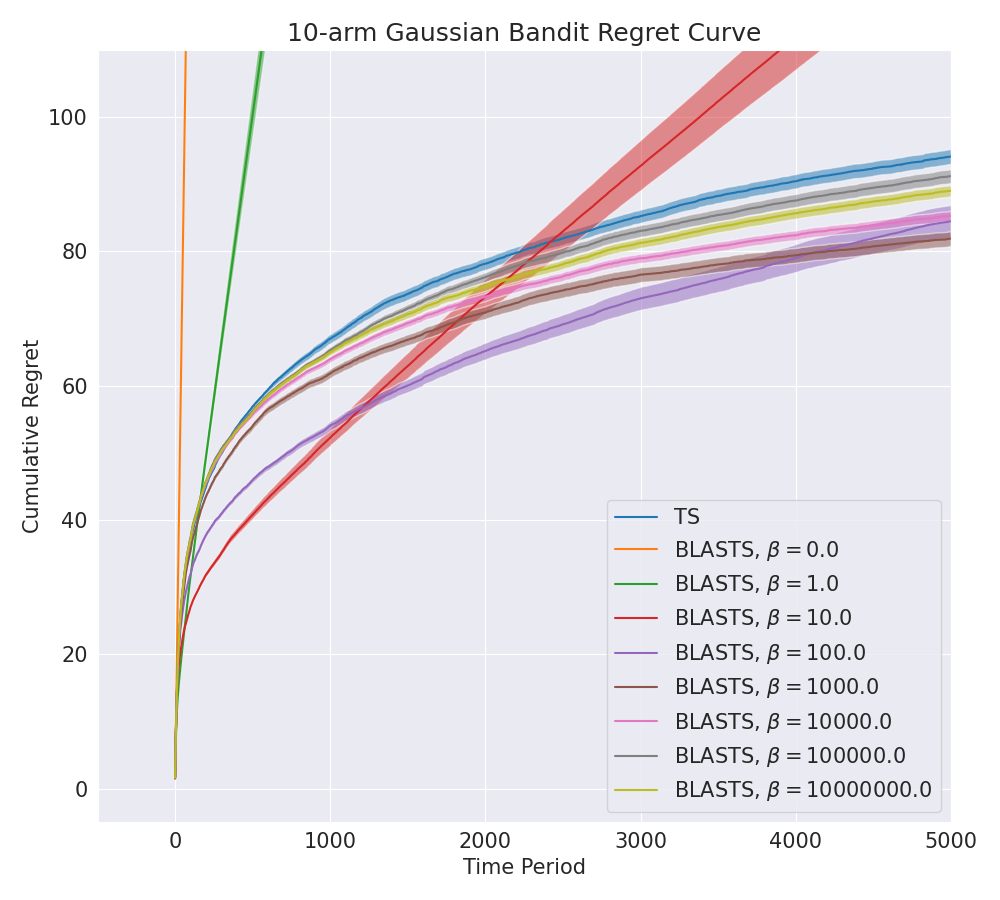}
\end{subfigure}
\caption{Cumulative regret curves for Bernoulli and Gaussian bandits with 10 independent arms comparing traditional Thompson Sampling (TS) against Blahut-Arimoto STS (BLASTS), sweeping over the $\beta$ hyperparameter of the latter.}
\label{fig:bandit_regret}
\end{figure}

While the previous experiments confirm that BLASTS can be used to instantiate a broad spectrum of agents that target actions of varying utilities, it is difficult to assess the simplicity of these targets and discern whether or not less-performant target actions can in fact be identified more quickly than near-optimal ones. As a starting point, one might begin with the agent's prior over the environment and compute $\bI_1(\mc{E}; \widetilde{A}_t)$ to quantify how much information each agent's initial learning target requires from the environment \textit{a priori}. In Figure \ref{fig:bandit_rd_curve}, we compare this to $\bI_1(\mc{E}; A_\eps)$ and sweep over the respective $\beta$ and $\eps$ values to generate the result rate-distortion curves for Bernoulli and Gaussian bandits with 1000 independent arms. The results corroborate earlier discussion of how a STS agent engages with a learning target $A_\eps$ that yields \textit{some} trade-off between ease of learnability and performance, but not necessarily the best trade-off. In contrast, since $\mc{R}_1(D) \approx \bI_1(\mc{E}; \widetilde{A}_t)$ (where the approximation is due to sampling), we expect and do indeed recover a better trade-off between rate and performance using the Blahut-Arimoto algorithm. To verify that target actions at the lower end of the spectrum (lower rate and higher distortion) can indeed by learned more quickly, we can plot the rate of the channel $\delta_t(\widetilde{A}_t \mid \mc{E})$ computed by BLASTS across time periods, as shown in Figure \ref{fig:bandit_rate}; for TS, we additionally plot the entropy over the optimal action $\bH_t(A^\star)$ as time passes and observe that smaller values of $\beta$ lead to learning targets with smaller initial rates that decay much more quickly than their counterparts at larger values of $\beta$. Again, as $\beta \uparrow \infty$, these rate curves concentrate around that of regular TS. 

\begin{figure}[H]
\centering
\begin{subfigure}{.5\textwidth}
  \centering
  \includegraphics[width=.85\linewidth]{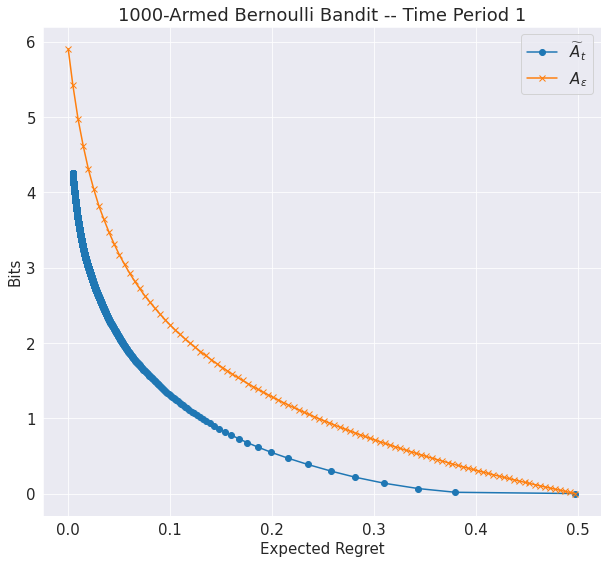}
\end{subfigure}%
\begin{subfigure}{.5\textwidth}
  \centering
  \includegraphics[width=.85\linewidth]{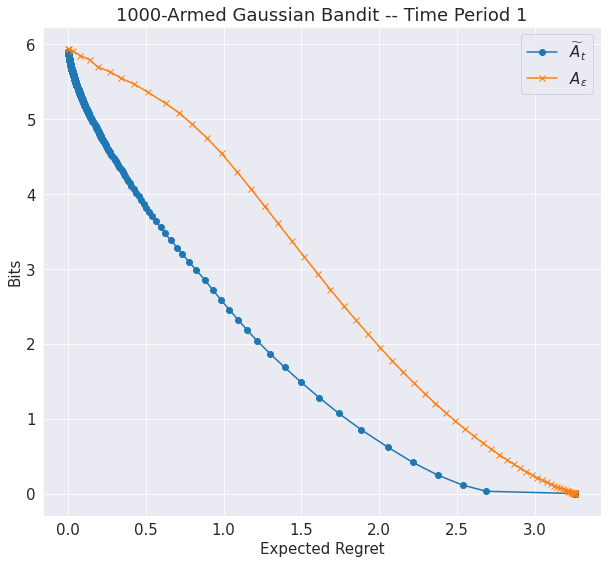}
\end{subfigure}
\caption{Rate-distortion curves for target actions computed via BLASTS ($\widetilde{A}_t$) and STS ($A_\eps$) in the first time periods of Bernoulli and Gaussian bandits with 1000 independent arms.}
\label{fig:bandit_rd_curve}
\end{figure}

\begin{figure}[H]
\centering
\begin{subfigure}{.5\textwidth}
  \centering
  \includegraphics[width=.85\linewidth]{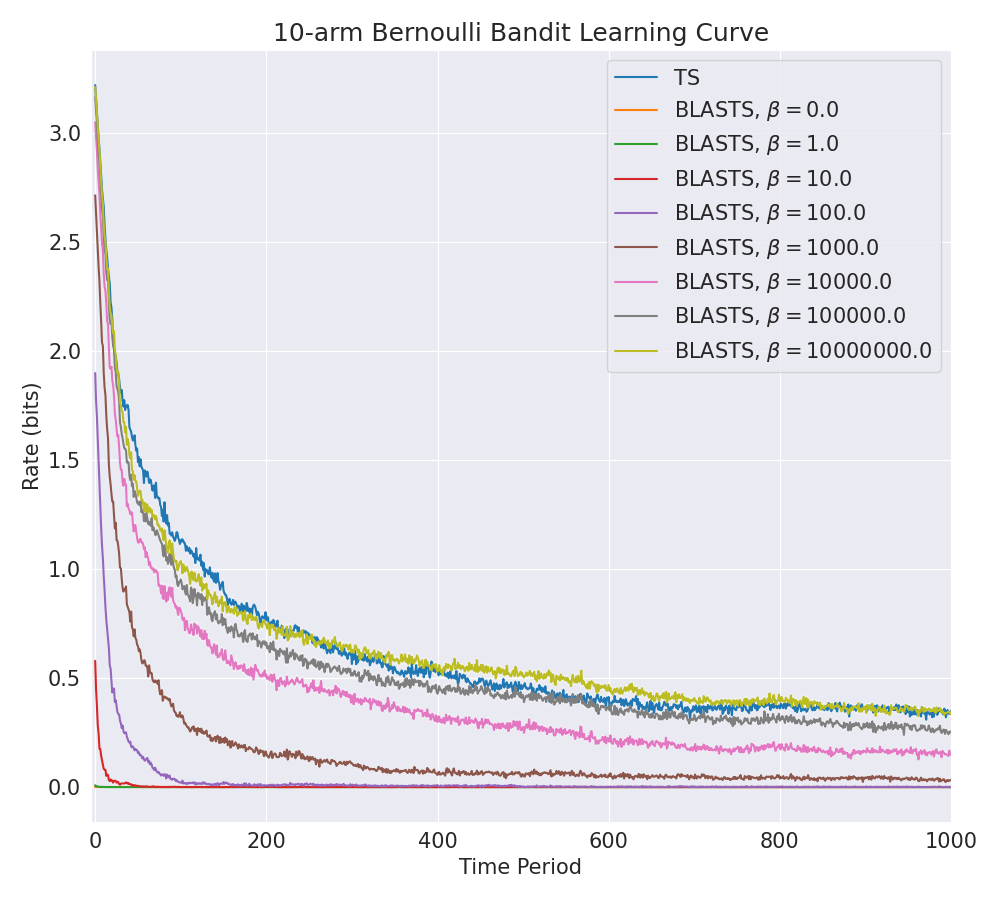}
\end{subfigure}%
\begin{subfigure}{.5\textwidth}
  \centering
  \includegraphics[width=.85\linewidth]{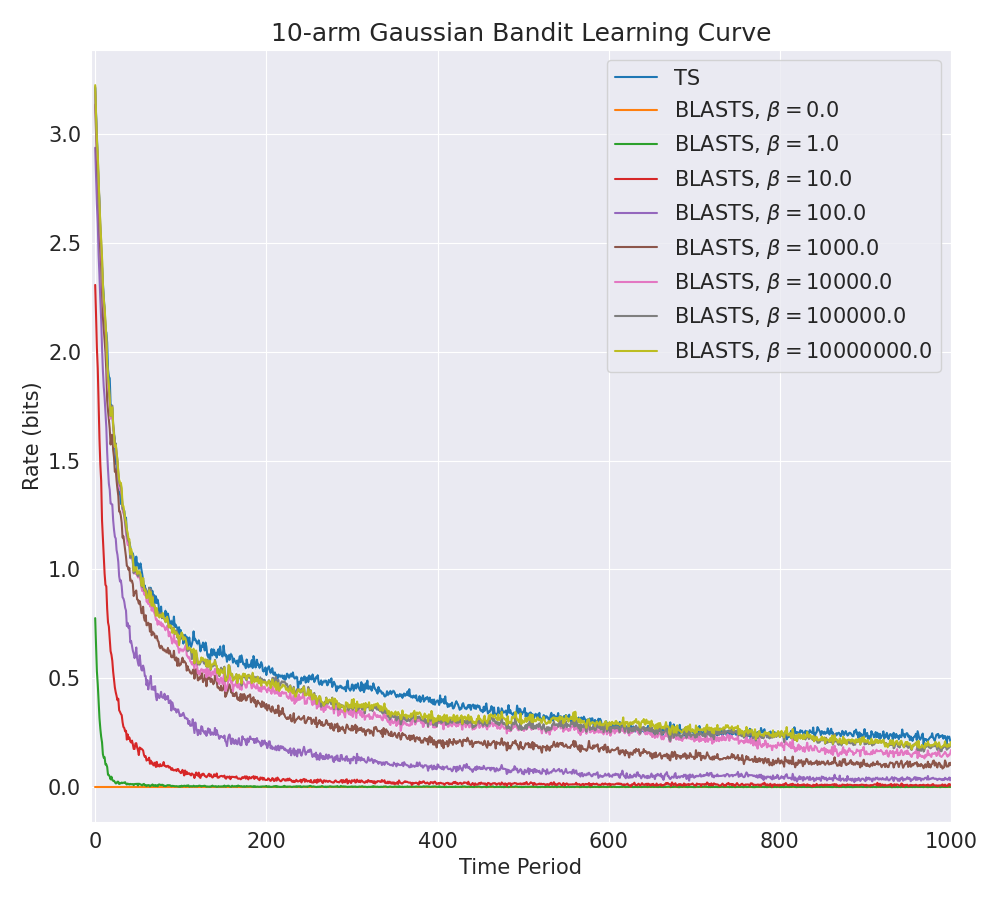}
\end{subfigure}
\caption{Rate curves for Bernoulli and Gaussian bandits with 10 independent arms comparing traditional Thompson Sampling (TS) against Blahut-Arimoto STS (BLASTS), sweeping over the $\beta$ hyperparameter of the latter.}
\label{fig:bandit_rate}
\end{figure}

Overall, this section has provided an overview of prior work that moves past the standard goal of finding optimal actions $A^\star$ in multi-armed bandit problems and towards capacity-limited decision-making agents. Extending beyond the empirical findings observed in these prior works, we provide additional experiments (see Figure \ref{fig:bandit_rate}) that show how the minimization of rate leads to target actions that are simpler to learn, allowing for an agent to curtail its interactions with the environment in fewer time periods and respect limitations on time and computational resources. Crucially, rate-distortion theory emerges as a natural conduit for identifying target actions that balance between respecting an agent's limits while still being sufficiently useful for the task at hand. In the next section, we extend this line of thinking to the episodic reinforcement-learning problem and survey recent theoretical results in this space that, analogous to Theorem \ref{thm:bandit_rdf_regret_bound} and Corollary \ref{thm:bandit_drf_regret_bound}, set the stage for subsequent empirical investigations into their practical veracity for both biological and artificial decision-making agents.

\subsection{Episodic Reinforcement Learning}

In this section, we again specialize the general problem formulation of Section \ref{sec:continual}, this time by introducing the assumption of episodicity commonly made throughout the reinforcement-learning literature. Just as in the preceding section, Thompson Sampling will again reappear as a quintessential algorithm for addressing exploration under an additional assumption that planning across any world model is always computationally feasible. Under this caveat, we survey existing theoretical results which accommodate capacity-limited agents via rate-distortion theory.

\subsubsection{Problem Formulation}

We formulate a sequential decision-making problem as an episodic, finite-horizon Markov Decision Process (MDP)~\citep{bellman1957markovian,Puterman94} defined by $\mc{M} = \langle \mc{S}, \mc{A}, \mc{U}, \mc{T}, \beta, H \rangle$. Here $\mc{S}$ denotes a set of states, $\mc{A}$ is a set of actions, $\mc{U}:\mc{S} \times \mc{A} \ra [0,1]$ is a deterministic reward or utility function providing evaluative feedback signals, $\mc{T}:\mc{S} \times \mc{A} \ra \Delta(\mc{S})$ is a transition function prescribing distributions over next states, $\beta \in \Delta(\mc{S})$ is an initial state distribution, and $H \in \bN$ is the maximum length or horizon. Within each one of $K \in \bN$ episodes, the agent acts for exactly $H$ steps beginning with an initial state $s_1 \sim \beta$. For each timestep $h \in [H]$, the agent observes the current state $s_h \in \mc{S}$, selects action $a_h \sim \pi_h(\cdot \mid s_h) \in \mc{A}$, enjoys a reward $r_h = \mc{U}(s_h,a_h) \in [0,1]$, and transitions to the next state $s_{h+1} \sim \mc{T}(\cdot \mid s_h, a_h) \in \mc{S}$.

A stationary, stochastic policy for timestep $h \in [H]$, $\pi_h:\mc{S} \ra \Delta(\mc{A})$, encodes behavior as a mapping from states to distributions over actions. Letting $\Pi \triangleq \{\mc{S} \ra \Delta(\mc{A})\}$ denote the class of all stationary, stochastic policies, a non-stationary policy $\pi = (\pi_1,\ldots,\pi_H) \in \Pi^H$ is a collection of exactly $H$ stationary, stochastic policies whose overall performance in any MDP $\mc{M}$ at timestep $h \in [H]$ when starting at state $s \in \mc{S}$ and taking action $a \in \mc{A}$ is assessed by its associated action-value function $Q^\pi_{\mc{M},h}(s,a) = \bE\left[\sum\limits_{h'=h}^H \mc{U}(s_{h'},a_{h'}) \bigm| s_h = s, a_h = a\right]$, where the expectation integrates over randomness in the action selections and transition dynamics. Taking the corresponding value function as $V^\pi_{\mc{M},h}(s) = \bE_{a \sim \pi_h(\cdot \mid s)}\left[Q^\pi_{\mc{M},h}(s,a)\right]$, we define the optimal policy $\pi^\star = (\pi^\star_1,\pi^\star_2,\ldots,\pi^\star_H)$ as achieving supremal value $V^\star_{\mc{M},h}(s) = \sup\limits_{\pi \in \Pi^H} V^\pi_{\mc{M},h}(s)$ for all $s \in \mc{S}$, $h \in [H]$. 
We let $\tau_k = (s^{(k)}_1, a^{(k)}_1, r^{(k)}_1, \ldots,s^{(k)}_{H}, a^{(k)}_{H}, r^{(k)}_{H}, s^{(k)}_{H+1})$ be the random variable denoting the trajectory experienced by the agent in the $k$th episode. Meanwhile, $H_k = \{\tau_1,\tau_2,\ldots, \tau_{k-1}\} \in \mc{H}_k$ is the random variable representing the entire history of the agent's interaction within the environment at the start of the $k$th episode. 

As is standard in Bayesian reinforcement learning~\citep{bellman1959adaptive,duff2002optimal,ghavamzadeh2015bayesian}, neither the transition function nor the reward function are known to the agent and, consequently, both are treated as random variables. An agent's initial uncertainty in the (unknown) true MDP $\mc{M}^\star = (\mc{U}^\star, \mc{T}^\star)$ is reflected by a prior distribution $p_1(\mc{M}^\star)$. As the agent's history of interaction within the environment unfolds, updated knowledge of the underlying MDP is reflected by posterior probabilities $p_k(\mc{M}^\star)$. Since the regret is a random variable due to our uncertainty in $\mc{M}^\star$, we integrate over this randomness to arrive at the Bayesian regret over $K$ episodes: $$\textsc{BayesRegret}(\{\pi^{(k)}\}_{k \in [K]}) = \bE\left[\textsc{Regret}(\{\pi^{(k)}\}_{k \in [K]}, \mc{M}^\star)\right] = \bE\left[\sum\limits_{k=1}^K \left( V^\star_{\mc{M}^\star,1}(s_1) - V^{\pi^{(k)}}_{\mc{M}^\star, 1}(s_1)\right)\right].$$

Just as in the previous section but with a slight abuse of notation, we will use $p_k(X) = p(X \mid H_k)$ as shorthand notation for the conditional distribution of any random variable $X$ given a random realization of an agent's history $H_k \in \mc{H}$, at any episode $k \in [K]$. Furthermore, we will denote the entropy and conditional entropy conditioned upon a specific realization of an agent's history $H_k$, for some episode $k \in [K]$, as $\bH_k(X) \triangleq \bH(X \mid H_k = H_k)$ and $\bH_k(X \mid Y) \triangleq \bH_k(X \mid Y, H_k = H_k)$, for two arbitrary random variables $X$ and $Y$. This notation will also apply analogously to mutual information: $\bI_k(X;Y) \triangleq \bI(X;Y \mid H_k = H_k) = \bH_k(X) - \bH_k(X \mid Y) = \bH_k(Y) - \bH_k(Y \mid X).$ We reiterate that a reader should interpret this as recognizing that, while standard information-theoretic quantities average over all associated random variables, an agent attempting to quantify information for the purposes of exploration does so not by averaging over all possible histories that it could potentially experience, but rather by conditioning based on the particular random history $H_k$ that it has currently observed thus far. The dependence on the realization of a random history $H_k$ makes $\bI_k(X;Y)$ a random variable and the usual conditional mutual information arises by integrating over this randomness: $\bE\left[\bI_k(X;Y)\right] = \bI(X;Y \mid H_k).$ Additionally, we will also adopt a similar notation to express a conditional expectation given the random history $H_k$: $\bE_k\left[X\right] \triangleq \bE\left[X|H_k\right].$

\subsubsection{Posterior Sampling for Reinforcement Learning}

A natural starting point for addressing the exploration challenge in a principled manner is via Thompson Sampling~\citep{thompson1933likelihood,russo2018tutorial}. The Posterior Sampling for Reinforcement Learning (PSRL)~\citep{strens2000bayesian,osband2013more,osband2014model,abbasi2014bayesian,agrawal2017optimistic,osband2017posterior,lu2019information} algorithm (given as Algorithm \ref{alg:psrl}) does this by, in each episode $k \in [K]$, sampling a candidate MDP $\mc{M}_k \sim p_k(\mc{M}^\star)$ and executing its optimal policy in the environment $\pi^{(k)} = \pi^\star_{\mc{M}_k}$; notably, such posterior sampling guarantees the hallmark probability-matching principle of Thompson Sampling: $p_k(\mc{M}_k = M) = p_k(\mc{M}^\star = M)$, $\forall M \in \mathfrak{M}, k \in [K]$. The resulting trajectory $\tau_k$ leads to a new history $H_{k+1} = H_k \cup \tau_k$ and an updated posterior over the true MDP $p_{k+1}(\mc{M}^\star)$. 

\begin{center}
\begin{minipage}{0.41\textwidth}
\vspace{-39pt}
\begin{algorithm}[H]
   \caption{Posterior Sampling for Reinforcement Learning (PSRL)~\citep{strens2000bayesian}}
   \label{alg:psrl}
\begin{algorithmic}
   \STATE {\bfseries Input:} Prior $p_1(\mc{M}^\star)$
   \FOR{$k \in [K]$}
   \STATE Sample $M_k \sim p_k(\mc{M}^\star)$
   \STATE Get optimal policy $\pi^{(k)} = \pi^\star_{M_k}$
   \STATE Execute $\pi^{(k)}$ and get trajectory $\tau_k$
   \STATE Update history $H_{k+1} = H_k \cup \tau_k$
   \STATE Induce posterior $p_{k+1}(\mc{M}^\star)$
   \ENDFOR
\end{algorithmic}
\end{algorithm}
\end{minipage}
\hfill
\begin{minipage}{0.58\textwidth}
\begin{algorithm}[H]
   \caption{Value-equivalent Sampling for Reinforcement Learning (VSRL)~\citep{arumugam2022deciding}}
   \label{alg:vsrl}
\begin{algorithmic}
   \STATE {\bfseries Input:} Prior $p_1(\mc{M}^\star)$, Threshold $D \in \bR_{\geq 0}$, Distortion function $d: \mathfrak{M} \times \mathfrak{M} \ra \bR_{\geq 0}$
   \FOR{$k \in [K]$}
   \STATE Compute $\widetilde{\mc{M}}_k$ achieving $\mc{R}_k(D)$ limit (Equation \ref{eq:target_mdp_rdf})
   \STATE Sample MDP $M^\star \sim p_k(\mc{M}^\star)$
   \STATE Sample compression $M_k \sim p(\widetilde{\mc{M}}_k \mid \mc{M}^\star = M^\star)$
   \STATE Compute optimal policy $\pi^{(k)} = \pi^\star_{M_k}$
   \STATE Execute $\pi^{(k)}$ and observe trajectory $\tau_k$
   \STATE Update history $H_{k+1} = H_k \cup \tau_k$
   \STATE Induce posterior $p_{k+1}(\mc{M}^\star)$
   \ENDFOR
\end{algorithmic}
\end{algorithm}
\end{minipage}
\end{center}

Unfortunately, for complex environments, pursuit of the exact MDP $\mc{M}^\star$ may be an entirely infeasible goal, akin to pursuing an optimal action $A^\star$ within a multi-armed bandit problem. A MDP representing control of a real-world, physical system, for example, suggests that learning the associated transition function requires the agent internalize laws of physics and motion with near-perfect accuracy. More formally, identifying $\mc{M}^\star$ demands the agent obtain exactly $\bH_1(\mc{M}^\star)$ bits of information from the environment which, under an uninformative prior, may either be prohibitively large by far exceeding the agent's capacity constraints or be simply impractical under time and resource constraints~\citep{lu2021reinforcement}. 

\subsubsection{Rate-Distortion Theory for Target MDPs}

To remedy the intractabilities imposed by PSRL when an agent must contend with an overwhelmingly-complex environment, we once again turn to rate-distortion theory as a tool for defining an information-theoretic surrogate than an agent may use to prioritize its information acquisition strategy in lieu of $\mc{M}^\star$. If one were to follow the rate-distortion optimization of Equation \ref{eq:continual_rdf}, this would suggest identifying a channel $\delta_t(\pi^{(k)} \mid \mc{M}^\star)$ that directly maps a bounded agent's beliefs about $\mc{M}^\star$ to a behavior policy $\pi^{(k)}$ for use in the current episode $k \in [K]$. For the purposes of analysis, \citet{arumugam2022deciding} instead perform lossy MDP compression with the interpretation that various facets of the true MDP $\mc{M}^\star$ must be discarded by a capacity-limited agent who can only hope identify a simplified world model that strives to retain as many salient details as possible. Implicit to such an approach is an assumption that the act of planning (that is, mapping any MDP $M \in \mathfrak{M}$ to its optimal policy $\pi^\star_M$) can always be done in a computationally-efficient manner irrespective of the agent's capacity limitations. From a mechanistic perspective, this is likely implausible for both artificial agents in large-scale, high-dimensional environments of interest as well as biological agents~\citep{ho2022people}. On the other hand, this construction induces a Markov chain $\mc{M}^\star - \widetilde{\mc{M}} - \pi^{(k)}$, where $\widetilde{\mc{M}}$ denotes the compressed world model; by the data-processing inequality, we have for all $k \in [K]$ that $\bI_k(\mc{M}^\star; \pi^{(k)}) \leq \bI_k(\mc{M}^\star; \widetilde{\mc{M}})$, such that minimizing the rate of the lossy MDP compression must also limit the amount of information that flows from the agent's beliefs about the world to the executed behavior policy.

For the precise details of this MDP compression, we first require (just as with any lossy compression problem) the specification of an information source to be compressed as well as a distortion function that quantifies the loss of fidelity between uncompressed and compressed values. Akin to the multi-armed bandit setting, we will take the agent's current beliefs $p_k(\mc{M}^\star)$ as the information source to be compressed in each episode. Unlike in the bandit setting, however, the choice of distortion function $d:\mathfrak{M} \times \mathfrak{M} \ra \bR_{\geq 0}$ presents an opportunity for the agent designer to be judicious in specifying which aspects of the environment are preserved in the agent's compressed view of the world.

It is fairly well accepted that human beings do not model all facets of the environment when making decisions~\citep{simon1956rational,gigerenzer1996reasoning} and the choice of which details are deemed salient enough to warrant retention in the mind of an agent is precisely governed by the choice of distortion function. In the computational reinforcement-learning literature, this reality has called into question longstanding approaches to model-based reinforcement learning~\citep{sutton1991dyna,sutton1998introduction,littman2015reinforcement} which use standard maximum-likelihood estimation techniques that endeavor to learn the exact model $(\mc{U},\mc{T})$ that governs the underlying MDP. The end result has been a flurry of recent work~\citep{silver2017predictron,farahmand2017value,oh2017value,asadi2018lipschitz,farahmand2018iterative,grimm2020value,d2020gradient,abachi2020policy,cui2020control,ayoub2020model,schrittwieser2020mastering,nair2020goal,grimm2021proper,nikishin2022control,voelcker2022value,grimm2022approximate} which eschews the traditional maximum-likelihood objective in favor of various surrogate objectives which restrict the focus of the agent's modeling towards specific aspects of the environment. As the core goal of endowing a decision-making agent with its own internal model of the world is to facilitate model-based planning~\citep{bertsekas1995dynamic}, central among these recent approaches is the value-equivalence principle~\citep{grimm2020value,grimm2021proper,grimm2022approximate} which provides mathematical clarity on how surrogate models can still enable lossless planning relative to the true model of the environment. 

For any arbitrary MDP $\mc{M}$ with model $(\mc{U},\mc{T})$ and any stationary, stochastic policy $\pi:\mc{S} \ra \Delta(\mc{A})$, define the Bellman operator $\mc{B}^\pi_\mc{M}: \{\mc{S} \ra \bR\} \ra \{\mc{S} \ra \bR\}$ as follows: $$\mc{B}^\pi_\mc{M}V(s) \triangleq \bE_{a \sim \pi(\cdot \mid s)}\left[\mc{U}(s,a) + \bE_{s' \sim \mc{T}(\cdot \mid s, a)}\left[ V(s')\right]\right].$$ The Bellman operator is a foundational tool in dynamic-programming approaches to reinforcement learning~\citep{bertsekas1995dynamic} and gives rise to the classic Bellman equation: for any MDP $\mc{M} = \langle \mc{S}, \mc{A}, \mc{U}, \mc{T}, \beta, H \rangle$ and any non-stationary policy $\pi = (\pi_1,\ldots,\pi_H)$, the value functions induced by $\pi$ satisfy $V^\pi_{\mc{M},h}(s) = \mc{B}^{\pi_h}_{\mc{M}}V^\pi_{\mc{M},h+1}(s),$ for all $h \in [H]$ and with $V^\pi_{\mc{M},H+1}(s) = 0$, $\forall s \in \mc{S}$. For any two MDPs $\mc{M} = \langle \mc{S}, \mc{A}, \mc{U}, \mc{T}, \beta, H \rangle$ and $\widehat{\mc{M}} = \langle \mc{S}, \mc{A}, \widehat{\mc{U}}, \widehat{\mc{T}}, \beta, H \rangle$, \citet{grimm2020value} define a notion of equivalence between them despite their differing models. For any policy class $\Pi \subseteq \{\mc{S} \ra \Delta(\mc{A})\}$ and value function class $\mc{V} \subseteq \{\mc{S} \ra \bR\}$, $\mc{M}$ and $\widehat{\mc{M}}$ are value equivalent with respect to $\Pi$ and $\mc{V}$ if and only if $\mc{B}^\pi_{\mc{M}}V = \mc{B}^\pi_{\widehat{\mc{M}}}V$, $\forall \pi \in \Pi, V \in \mc{V}.$ In words, two different models are deemed value equivalent if they induce identical Bellman updates under any pair of policy and value function from $\Pi \times \mc{V}$. \citet{grimm2020value} prove that when $\Pi = \{\mc{S} \ra \Delta(\mc{A})\}$ and $\mc{V} = \{\mc{S} \ra \bR\}$, the set of all exactly value-equivalent models is a singleton set containing only the true model of the environment. By recognizing that the ability to plan over all arbitrary behaviors is not necessarily in the agent's best interest and restricting focus to decreasing subsets of policies $\Pi \subset \{\mc{S} \ra \Delta(\mc{A})\}$ and value functions $\mc{V} \subset \{\mc{S} \ra \bR\}$, the space of exactly value-equivalent models is monotonically increasing. 

Still, however, exact value equivalence still presumes that an agent has the capacity for planning with complete fidelity to the true environment; more plausibly, an agent may only have the resources to plan in an approximately-value-equivalent manner~\citep{grimm2022approximate}. For brevity, let $\mathfrak{R} \triangleq \{\mc{S} \times \mc{A} \ra [0,1]\}$ and $\mathfrak{T} \triangleq \{\mc{S} \times \mc{A} \ra \Delta(\mc{S})\}$ denote the classes of all reward functions and transition functions, respectively. Recall that, with $\langle \mc{S}, \mc{A}, \beta, H \rangle$ all known, the uncertainty in a random MDP $\mc{M}$ is entirely driven by its model $(\mc{R},\mc{T})$ such that we may think of the support of $\mc{M}^\star$ as $\text{supp}(\mc{M}^\star) = \mathfrak{M} \triangleq \mathfrak{R} \times \mathfrak{T}$. We define a distortion function on pairs of MDPs $d:\mathfrak{M} \times \mathfrak{M} \ra \bR_{\geq 0}$ for any $\Pi \subseteq \{\mc{S} \ra \Delta(\mc{A})\}$, $\mc{V} \subseteq \{\mc{S} \ra \bR\}$ as $$d_{\Pi,\mc{V}}(\mc{M},\widehat{\mc{M}}) = \sup\limits_{\substack{\pi \in \Pi \\ V \in \mc{V}}} ||\mc{B}^\pi_{\mc{M}}V - \mc{B}^\pi_{\widehat{\mc{M}}}V||_\infty^2 = \sup\limits_{\substack{\pi \in \Pi \\ V \in \mc{V}}} \left(\sup\limits_{s \in \mc{S}} |\mc{B}^\pi_{\mc{M}}V(s) - \mc{B}^\pi_{\widehat{\mc{M}}}V(s)| \right)^2.$$ In words, $d_{\Pi,\mc{V}}$ is the supremal squared Bellman error between MDPs $\mc{M}$ and $\widehat{\mc{M}}$ across all states $s \in \mc{S}$ with respect to the policy class $\Pi$ and value function class $\mc{V}$. With an information source and distortion function defined, \citet{arumugam2022deciding} employ the following rate-distortion function that articulates the lossy MDP compression a capacity-limited decision agent performs to identify a simplified MDP to pursue instead of $\mc{M}^\star$: 
\begin{align}
    \mc{R}_k(D) &= \inf\limits_{p(\widetilde{\mc{M}} \mid \mc{M}^\star)} \bI_k(\mc{M}^\star; \widetilde{\mc{M}}) \text{ such that } \bE_k[d(\mc{M}^\star, \widetilde{\mc{M}})] \leq D.
    \label{eq:target_mdp_rdf}
\end{align}
By definition, the target MDP $\widetilde{\mc{M}}_k$ that achieves this rate-distortion limit will demand that the agent acquire fewer bits of information than what is needed to identify $\mc{M}^\star$. Once again, by virtue of Fact \ref{fact:rdf}, this claim is guaranteed for all $k \in [K]$ and any $D > 0$: $\mc{R}_k(D) \leq \mc{R}_k(0) \leq \bI_k(\mc{M}^\star; \mc{M}^\star) = \bH_k(\mc{M}^\star)$. Crucially, however, the use of the value-equivalence principle in the distortion function ensures that agent capacity is allocated towards preserving the regions of the world model needed to plan over behaviors as defined through $\Pi,\mc{V}$. \citet{arumugam2022deciding} establish an information-theoretic Bayesian regret bound for a posterior-sampling algorithm (given as Algorithm \ref{alg:vsrl}) that performs probability matching with respect to $\widetilde{\mc{M}}_k$ in each episode $k \in [K]$, instead of $\mc{M}^\star$: $\textsc{BayesRegret}(\{\pi^{(k)}\}_{k \in [K]}) \leq \sqrt{\overline{\Gamma}K\mc{R}_1(D)} + 2KH\sqrt{D}$, where $\overline{\Gamma} < \infty$ is an uniform upper bound to the information ratio~\citep{russo2016information,russo2014learning,russo2018learning} that emerges as a technical assumption for the analysis; a reader should interpret this $\overline{\Gamma}$ as a sort of conversion factor communicating the worst case number of units of squared regret incurred by the agent per bit of information acquired from the environment. 

Just as with the BLASTS algorithm for the multi-armed bandit setting, this VSRL algorithm directly couples an agent's exploratory choices in each episode to the epistemic uncertainty it maintains over the resource-rational learning target $\widetilde{\mc{M}}_k$ which it aspires to learn. The bound communicates that an agent with limited capacity must tolerate a higher distortion threshold $D$ and pursue the resulting compressed MDP that bears less fidelity to the original MDP; in exchange, the resulting number of bits needed from the environment to identify such a simplified model of the world is given as $\mc{R}_1(D)$ and guaranteed to be less than the entropy of $\mc{M}^\star$. Additionally, just as with the regret bound for BLASTS, one can express a near-identical result through the associated distortion-rate function. In particular, this encourages a particular notion of agent capacity as a limit $R \in \bR_{\geq 0}$ on the number of bits an agent may obtain from its interactions with the environment. Subject to this constraint, the fundamental limit on the amount of expected distortion incurred is given by
\begin{align}
    \mc{D}_t(R) = \inf\limits_{p(\widetilde{\mc{M}} \mid \mc{M}^\star)} \bE_k[d(\mc{M}^\star, \widetilde{\mc{M}})] \text{ such that } \bI_k(\mc{M}^\star; \widetilde{\mc{M}}) \leq R.
    \label{eq:target_mdp_drf}
\end{align}
Embracing this distortion-rate function and taking the VSRL distortion threshold as $D = \mc{D}_1(R)$ allows for a performance guarantee 
that explicitly accounts for the agent capacity limits: $\textsc{BayesRegret}(\{\pi^{(k)}\}_{k \in [K]}) \leq \sqrt{\overline{\Gamma}KR} + 2KH\sqrt{\mc{D}_1(R)}.$

In summary, under a technical assumption of episodicity for the purposes of analysis, the theoretical results surveyed in this section parallel those of the preceding section for multi-armed bandits. While computational experiments for this episodic reinforcement learning setting have not yet been established due to the computational efficiency of running the Blahut-Arimoto algorithm for such a lossy MDP compression problem, the core takeaway of this section is that there is strong theoretical justification for using these tools from rate-distortion theory to empirically study capacity-limited sequential decision-making agents.

\section{Discussion}
\label{sec:disc}

In this paper, we have introduced capacity-limited Bayesian reinforcement learning, capturing a novel perspective on lifelong learning under a limited cognitive load while also surveying existing theoretical and algorithmic advances specific to multi-armed bandits~\citep{arumugam2021deciding} and reinforcement learning~\citep{arumugam2022deciding}. Taking a step back, we now situate our contributions in a broader context by reviewing related work on capacity-limited cognition as well as information-theoretic reinforcement learning. As our framework sits at the intersection of Bayesian inference, reinforcement learning, and rate-distortion theory, we use this opportunity to highlight particularly salient pieces of prior work that sit at the intersection Bayesian inference and rate-distortion theory as well as the intersection of reinforcement learning and rate-distortion theory, respectively. Furthermore, while the algorithms discussed in this work all operationalize the Blahut-Arimoto algorithm and Thompson Sampling as the primary mechanisms for handling rate-distortion optimization and exploration respectively, we also discuss opportunities to expand to more sophisticated strategies for computing target actions and exploring once a target action has been determined. Lastly, we conclude our discussion by returning to a key assumption used throughout this work that an agent consistently maintains idealized beliefs about the environment $\mc{E}$ through perfect Bayesian inference. 

\subsection{Related Work on Learning, Decision-Making, and Rate-Distortion Theory}



There is a long, rich literature exploring the natural limitations on time, knowledge, and cognitive capacity faced by human (and animal) decision makers~\citep{simon1956rational,newell1958elements,newell1972human,shugan1980cost,simon1982models,gigerenzer1996reasoning,vul2014one,griffiths2015rational,gershman2015computational,icard2015resource,lieder2020resource,amir2020value,bhui2021resource,brown2022humans,ho2022people,prystawski2022resource,binz2022modeling}. Crucially, our focus is on a recurring theme throughout this literature of modeling these limitations on cognitive capabilities as being information-theoretic in nature~\citep{sims2003implications,peng2005learning,parush2011dopaminergic,botvinick2015reinforcement,sims2016rate,sims2018efficient,zenon2019information,ho2020efficiency,gershman2020reward,gershman2020origin,mikhael2021rational,lai2021policy,gershman2021rational,jakob2022rate,bari2022undermatching}. 

Broadly speaking and under the episodic reinforcement learning formulation of the previous section, these approaches all center around the perspective that a policy $\pi_h: \mc{S} \ra \Delta(\mc{A})$ should be modeled as a communication channel that, like a human decision-maker with limited information processing capability, is subject to a constraint on the maximal number of bits that may be sent across it. Consequently, an agent aspiring to maximize returns must do so subject to this constraint on policy complexity; conversely, an agent ought to transmit the minimum amount of information possible while it endeavors to reach a desired level of performance~\citep{polani2009information,polani2011informational,tishby2011information,rubin2012trading}. Paralleling the distortion-rate function $\mc{D}(R)$, the resulting policy-optimization objective follows as $\sup\limits_{\pi \in \{\mc{S} \ra \Delta(\mc{A})\}^H} \bE\left[Q^\pi(S, A)\right] \text{ such that } \bI(S; A) \leq R.$ It is important to acknowledge that such a formulation sits directly at the intersection of reinforcement learning and rate-distortion theory without invoking any principles of Bayesian inference. Depending on the precise work, subtle variations on this optimization problem exist from choosing a fixed state distribution for the random variable $S$~\citep{polani2009information,polani2011informational},  incorporating the state visitation distribution of the policy being optimized~\citep{still2012information,gershman2020origin,lai2021policy}, or assuming access to the generative model of the MDP and decomposing the objective across a finite state space~\citep{tishby2011information,rubin2012trading}. In all of these cases, the end empirical result tends to converge by also making use of variations on the classic Blahut-Arimoto algorithm to solve the Lagrangian associated with the constrained optimization~\citep{boyd2004convex} and produce policies that exhibit higher entropy across states under an excessively limited rate $R$, with a gradual convergence towards the greedy optimal policy as $R$ increases. 

The alignment between this optimization problem and that of the distortion-rate function is slightly wrinkled by the non-stationarity of the distortion function (here, $Q^\pi$ is used as an analogue to distortion which changes as the policy or channel does) and, when using the policy visitation distribution for $S$, the non-stationarity of the information source. Despite these slight, subtle mismatches with the core rate-distortion problem, the natural synergy between cognitive and computational decision making~\citep{tenenbaum2011grow,lake2017building} has led to various reinforcement-learning approaches that draw direct inspiration from this line of thinking~\citep{klyubin2005empowerment,ortega2011information,still2012information,ortega2013thermodynamics,shafieepoorfard2016rationally,tiomkin2017unified,goyal2018infobot,lerch2018policy,lerch2019rate,abel2019state,goyal2020variational,goyal2020reinforcement}, most notably including parallel connections to work on ``control as inference'' or KL-regularized reinforcement learning~\citep{todorov2007linearly,toussaint2009robot,kappen2012optimal,levine2018reinforcement,ziebart2010modeling,fox2016taming,haarnoja2017reinforcement,haarnoja2018soft,galashov2019information,tirumala2019exploiting}. Nevertheless, despite their empirical successes, such approaches lack principled mechanisms for addressing the exploration challenge~\citep{o2020making}. In short, the key reason behind this is that the incorporation of Bayesian inference allows for a separation of reducible or epistemic uncertainty that exists due to an agent's lack of knowledge versus irreducible or aleatoric uncertainty that exists due to the natural stochasticity that may exist within a random outcome~\citep{der2009aleatory}. Without leveraging a Bayesian setting, a random variable denoting an agent's belief about the environment $\mc{E}$ or underlying MDP $\mc{M}^\star$ no longer exists and a channel like the ones explored throughout this work from beliefs to action cease to exist. That said, the notion of rate preserved by these methods has been shown to constitute a reasonable notion of policy complexity~\citep{lai2021policy} and future work may benefit from combining the two approaches.

Similar to human decision making~\citep{gershman2018deconstructing,schulz2019algorithmic,gershman2019uncertainty}, provably-efficient reinforcement-learning algorithms have historically relied upon one of two possible exploration strategies: optimism in the face of uncertainty~\citep{kearns2002near,brafman2002r,kakade2003sample,auer2009near,bartlett2009regal,strehl2009reinforcement,jaksch2010near,dann2015sample,azar2017minimax,dann2017unifying,jin2018q,zanette2019tighter,dong2022simple} or posterior sampling~\citep{osband2013more,osband2017posterior,agrawal2017optimistic,lu2019information,lu2021reinforcement}. While both paradigms have laid down solid theoretical foundations, a line of work has demonstrated how posterior-sampling methods can be more favorable both in theory and in practice~\citep{osband2013more,osband2016deep,osband2016generalization,osband2017posterior,osband2019deep,dwaracherla2020hypermodels}. The theoretical results discussed in this work advance and further generalize this line of thinking through the concept of \textit{learning targets} (referred to in this work as target actions for clarity of exposition), introduced by \citet{lu2021reinforcement}, which opens up new avenues for entertaining solutions beyond optimal policies and conditioning an agent's exploration based on what it endeavors to learn from its environment, not unlike preschool children~\citep{cook2011science}. While this literature traditionally centers on consideration of a single agent interacting within its environment, generalizations to multiple agents acting concurrently while coupled through shared beliefs have been formalized and examined in theory as well as in practice~\citep{dimakopoulou2018coordinated,dimakopoulou2018scalable,chen2022society}; translating the ideas discussed here to further account for capacity limitations in that setting constitutes a promising direction for future work. 

Finally, we note while the work cited thus far was developed in the reinforcement learning community, the coupling of rate-distortion theory and Bayesian inference to strike a balance between the simplicity and utility of what an agent learns has been studied extensively by \citet{gottwald2019bounded}, who come from an information-theoretic background studying bounded rationality~\citep{ortega2011information,ortega2013thermodynamics}. Perhaps the key distinction between the work surveyed here and theirs is the further incorporation of reinforcement learning, which then provides a slightly more precise foundation upon which existing machinery can be repurposed to derive theoretical results like regret bounds. In contrast, the formulation of \citet{gottwald2019bounded} follows more abstract utility-theoretic decision making while also leveraging ideas from microeconomics and generalized beyond from standard Shannon information-theoretic quantities; we refer readers to their excellent, rigorous treatment of this topic.



\subsection{Generalizations to Other Families of Decision Rules}

The previous sections demonstrated several concrete implementations of capacity-limited Bayesian decision-making. We focused on BLASTS, an algorithm that generalizes Thompson Sampling, which itself is already a quintessential algorithm for navigating the explore-exploit tradeoff in a principled manner in multi-armed bandit and sequential decision-making problems. That said, however, we emphasize that BLASTS is only one particular instantiation of the framework espoused by the rate-distortion function of Equation \ref{eq:continual_rdf}. Here, we briefly sketch other directions in which the framework has been or could be applied.

First, the general framework of capacity-limited Bayesian decision-making can, in principle, be applied to any algorithm that, when supplied with beliefs about the environment and a particular target for learning, induces a policy to execute in the environment. For example, in \emph{information-directed sampling}, choices are made not only based on current beliefs about immediate rewards but also based on how actions produce informative consequences that can guide future behavior~\citep{russo2014learning,russo2018learning,lu2021reinforcement,hao2022contextual,hao2022regret}. This strategy motivates a decision-maker to engage in \emph{direct exploration} as opposed to \emph{random exploration} (Thompson Sampling being one example)~\citep{wilson2014humans} and better resolve the explore-exploit dilemma. Work by~\citet{arumugam2021the} has extended the BLASTS algorithm to develop variants of information-directed sampling that similarly minimize the rate between environment estimates and actions. Future work could explore even richer families of decision-rules such as those based on Bayes-optimal solutions over longer time horizons~\citep{duff2002optimal} and even ones that look past the KL-divergence as the core quantifier of information~\citep{lattimore2019information,zimmert2019connections,lattimore2021mirror}.

Additionally, BLASTS itself uses a seminal algorithm from the information-theory literature to ultimately address the rate-distortion optimization problem and find the decision-rule that optimally trades off reward and information---namely, the Blahut-Arimoto algorithm~\citep{blahut1972computation,arimoto1972algorithm}. However, this standard algorithm, while mathematically sound for random variables taking values on abstract spaces~\citep{csiszar1974extremum}, can only be made computationally tractable in the face of discrete random variables. Extending to general \emph{input} distributions (\textit{e.g.}, distributions with continuous or countable support) occurs through the use of an estimator with elegant theoretical properties such as asymptotic consistency~\citep{harrison2008estimation,palaiyanur2008uniform}. Despite this, it is still limited to \emph{output} distributions that have finite support. This limits its applicability to problems where the action space is finite and relatively small (even if the environment space is complex). Thus, an important direction for future research will be to develop algorithms for finding capacity-limited decision-rules based on versions of Blahut-Arimoto designed for general output distributions (\textit{e.g.}, particle filter-based algorithms~\citep{dauwels2005numerical}).

\subsection{Capacity-Limited Estimation and Alternative Information Bottlenecks}
Throughout this paper, we have assumed that environment estimation is not directly subject to capacity-limitations and that decision-makers perform perfect Bayesian inference. Naturally, however, this idealized scenario isn't guaranteed to hold for biological or artificial decision making agents. One high-level perspective on the core agent design problem addressed in this work is that decision-making agents cannot acquire unbounded quantities of information from the environment -- this reality motivates the need to prioritize information and rate-distortion theory emerges as a natural tool for facilitating such a prioritization scheme. 

By the same token, capacity-limited decision-making agents should also seldom find themselves capable of \textit{retaining} all bits of information uncovered about the underlying environment $\mc{E}$. If this were possible, then maintaining perfect belief estimates about the environment via $\eta_t$ would be a reasonable supposition. In reality, however, an agent must also be judicious in what pieces of environment information are actually retained. \citet{lu2021reinforcement} introduce terminology for discussing this limited corpus of world knowledge as an \textit{environment proxy}, $\widetilde{\mc{E}}$. The lack of fidelity between this surrogate and true environment $\mc{E}$ translates to the approximate nature of an agent's Bayesian inference when maintaining beliefs about $\widetilde{\mc{E}}$ in lieu of $\mc{E}$. For biological decision-making agents, the concept of a proxy seems intuitive as ``we are not interested in describing some physically objective world in its totality, but only those aspects of the totality that have relevance as the `life space' of the organism considered. Hence, what we call the `environment' will depend upon the `needs,' `drives,' or `goals' of the organism,'' as noted by Herbert Simon many decades ago~\citep{simon1956rational}.

Curiously, the relationship between the original environment $\mc{E}$ and this proxy $\widetilde{\mc{E}}$ can also be seen as a lossy compression problem where only a salient subset of the cumulative environment information need by retained by the agent for competent decision-making. Consequently, the associated 
rate-distortion function and the question of what suitable candidate notions of distortion apply may likely be an interesting object of study for future work. Practical optimization of such a rate-distortion function would likely benefit from recent statistical advances in empirical distribution compression~\citep{dwivedi2022generalized} to get away with representing the information source via a limited number of Monte-Carlo samples.

Finally, although consideration of capacity-limits on inference would extend the scope of the current framework, it is worth noting that recent findings in neuroscience support the possibility of a bottleneck on choice processes even if the bottleneck on inference is minimal. For example, when trained on stimuli presented at different angles, mice have been shown to discriminate orientations as low as $20^\circ$-$30^\circ$ based on \emph{behavioral} measures~\citep{abdolrahmani2019cognitive}. However, direct \emph{neural} measurements from visual processing regions reveal sensitivity to orientations as low as $0.37^\circ$~\citep{stringer2021high}. The higher precision (nearly $100\times$ higher) of sensory versus behavioral discrimination is consistent with a greater information bandwidth on inference compared to choice, as assumed in the current version of the model.
Similarly, work tracking the development of decision-making strategies in children provides evidence of capacity limits on choice processes even in the absence of limits on inference. For example, \cite{decker2016creatures} report that on a task designed to dissociate model-free versus model-based learning mechanisms, 8-12 year olds show signs of encoding changes in transition structure (longer reaction times) but do not appear to use this information to make better decisions, unlike 13-17 year olds and adults. This result is consistent with a distinct bottleneck between inference and action that may have a developmental trajectory. In short, the analyses developed in this paper may shed light on the general computational principles that underlie cases in which decision-makers display optimal inference but suboptimal choice. 

\subsection{Conclusion}
\label{sec:conc}
Our goal in this paper has been to review key insights from work on capacity-limited Bayesian decision-making by~\cite{arumugam2021deciding,arumugam2022deciding} and situate it within existing work on resource-rational cognition and decision-making~\citep{griffiths2015rational,lieder2020resource,gershman2015computational}. This discussion naturally leads to a number of questions, in particular, how the general framework presented can be applied to a wider range of algorithms, how other kinds of information bottlenecks could affect learning, and whether humans and other animals are capacity-limited Bayesian decision-makers. We hope that by formally outlining the different components of capacity-limited inference and choice, the current work can facilitate future cross-disciplinary investigations to address such topics.








\bibliographystyle{plainnat}
\bibliography{references}


\newpage
\appendix

\section{Proof of Theorem \ref{thm:bandit_rdf_regret_bound}}

We begin our analysis of Rate-Distortion Thompson Sampling by establishing the following fact, which also appears in the proof of Lemma 3 of \citep{arumugam2021deciding}:

\begin{fact}
For any target action $\tilde{A}$ and any time period $t \in [T]$, $$\bI_t(\tilde{A}; (A_t, O_{t+1})) = \bI_t(\mc{E}; \tilde{A}) - \bI_t(\mc{E}; \tilde{A} \mid A_t, O_{t+1}).$$
\label{fact:info_gain}
\end{fact}
\begin{proof}
Recall that for any $t \in [T]$, $H_{t+1} = (H_t, A_t, O_{t+1})$. Moreover, no action-observation pair can offer more information about any target action $\tilde{A}$ than the environment $\mc{E}$ itself. Thus, we have that $\forall t \in [T]$, $H_t \perp \tilde{A} \mid \mc{E}$, which implies $\bI_t(\tilde{A}; (A_t, O_{t+1}) \mid \mc{E}) = 0$. By the chain rule of mutual information, $$\bI_t(\mc{E}; \tilde{A}) = \bI_t(\mc{E}; \tilde{A}) + \bI_t(\tilde{A}; (A_t, O_{t+1}) \mid \mc{E}) = \bI_t(\mc{E}, (A_t, O_{t+1}); \tilde{A}).$$ Applying the chain rule of mutual information a second time yields $$\bI_t(\mc{E}; \tilde{A}) = \bI_t(\mc{E}, (A_t, O_{t+1}); \tilde{A}) = \bI_t(\tilde{A}; (A_t, O_{t+1})) + \bI_t(\mc{E}; \tilde{A} \mid A_t, O_{t+1}).$$ Finally, simply re-arranging terms gives $$\bI_t(\tilde{A}; (A_t, O_{t+1})) = \bI_t(\mc{E}; \tilde{A}) - \bI_t(\mc{E}; \tilde{A} \mid A_t, O_{t+1}),$$ as desired.
\end{proof}

\begin{lemma}
For any $D > 0$ and all $t \in [T]$, $$\bE_t\left[\mc{R}_{t+1}(D)\right] \leq \mc{R}_t(D) - \bI_t(\tilde{A}_t; (A_t, O_{t+1})).$$
\label{lemma:exp_rdf_ub}
\end{lemma}
\begin{proof}
By definition, $\tilde{A}_t$ achieves the rate-distortion limit such that $\bE_t\left[d(\tilde{A}_t, \mc{E})\right] \leq D.$ Recall that, by Fact \ref{fact:rdf}, the rate-distortion function is a non-increasing function in its argument. This implies that for any $D_1 \leq D_2$, $\mc{R}_{t+1}(D_2) \leq \mc{R}_{t+1}(D_1)$. Applying this fact to the inequality above and taking expectations, we obtain $$\bE_t\left[\mc{R}_{t+1}(D)\right] \leq \bE_t\left[\mc{R}_{t+1}\left(\bE_t\left[d(\tilde{A}_t, \mc{E})\right]\right)\right].$$
Observe by the tower property of expectation that $$\bE_t\left[\mc{R}_{t+1}(D)\right] \leq \bE_t\left[\mc{R}_{t+1}\left(\bE_t\left[d(\tilde{A}_t, \mc{E})\right]\right)\right] = \bE_t\left[\mc{R}_{t+1}\left(\bE_t\left[\bE_{t+1}\left[d(\tilde{A}_t, \mc{E})\right]\right]\right)\right].$$
Moreover, from Fact \ref{fact:rdf}, we recall that the rate-distortion function is a convex function. Consequently, by Jensen's inequality, we have
$$\bE_t\left[\mc{R}_{t+1}(D)\right] \leq \bE_t\left[\mc{R}_{t+1}\left(\bE_t\left[d(\tilde{A}_t, \mc{E})\right]\right)\right] = \bE_t\left[\mc{R}_{t+1}\left(\bE_t\left[\bE_{t+1}\left[d(\tilde{A}_t, \mc{E})\right]\right]\right)\right] \leq \bE_t\left[\mc{R}_{t+1}\left(\bE_{t+1}\left[d(\tilde{A}_t, \mc{E})\right]\right)\right].$$
Inspecting the definition of the rate-distortion in the expectation, we see that 
$$\mc{R}_{t+1}(D) = \inf\limits_{p(\widetilde{A} \mid \mc{E})} \bI_{t+1}(\mc{E}; \widetilde{A}) \text{ such that } \bE_{t+1}\left[d(\widetilde{A}, \mc{E})\right] \leq D,$$ which immediately implies $$\mc{R}_{t+1}\left(\bE_{t+1}\left[d(\tilde{A}_t, \mc{E})\right]\right) \leq \bI_{t+1}(\mc{E}; \tilde{A}_t).$$ Substituting back into the earlier expression, we have
$$\bE_t\left[\mc{R}_{t+1}(D)\right] \leq \bE_t\left[\bI_{t+1}(\mc{E}; \tilde{A}_t)\right] = \bE_t\left[\bI_{t}(\mc{E}; \tilde{A}_t \mid A_t, O_{t+1})\right] = \bI_{t}(\mc{E}; \tilde{A}_t \mid A_t, O_{t+1}).$$
We now apply Fact \ref{fact:info_gain} to arrive at $$\bE_t\left[\mc{R}_{t+1}(D)\right] \leq \bI_{t}(\mc{E}; \tilde{A}_t \mid A_t, O_{t+1}) = \bI_{t}(\mc{E}; \tilde{A}_t) - \bI_t(\tilde{A}_t; (A_t, O_{t+1})).$$

Since, by definition, $\tilde{A}_t$ achieves the rate-distortion limit at time period $t$, we know that $\bI_t(\mc{E}; \tilde{A}_t) = \mc{R}_t(D).$ Applying this fact yields the desired inequality: $$\bE_t\left[\mc{R}_{t+1}(D)\right] \leq \bI_{t}(\mc{E}; \tilde{A}_t) - \bI_t(\tilde{A}_t; (A_t, O_{t+1})) = \mc{R}_t(D) - \bI_t(\tilde{A}_t; (A_t, O_{t+1})).$$
\end{proof}

Lemma \ref{lemma:exp_rdf_ub} shows that the expected amount of information needed from the environment in each successive time period is non-increasing and further highlights two possible sources for this improvement: (1) a change in learning target from $\tilde{A}_t$ to $\tilde{A}_{t+1}$ and (2) information acquired about $\tilde{A}_t$ in the current time period, $\bI_t(\tilde{A}_t; (A_t,O_{t+1}))$. With this in hand, we can obtain control over the cumulative information gain of an agent across all time periods using the learning target identified under our prior, following an identical argument as \citet{arumugam2022deciding}.

\begin{lemma}
For any fixed $D > 0$ and any $t \in [T]$, $$\bE_t\left[\sum\limits_{t'=t}^T \bI_{t'}(\tilde{A}_{t'}; (A_{t'}, O_{t'+1}))\right] \leq \mc{R}_t(D).$$
\label{lemma:cum_info_bound}
\end{lemma}
\begin{proof}
Observe that we can apply Lemma \ref{lemma:exp_rdf_ub} directly to each term of the sum and obtain $$\bE_t\left[\sum\limits_{t'=t}^T \bI_{t'}(\tilde{A}_{t'}; (A_{t'}, O_{t'+1}))\right] \leq \bE_t\left[\sum\limits_{t'=t}^T \left(\mc{R}_{t'}(D) - \bE_{t'}\left[\mc{R}_{t'+1}(D)\right]\right) \right].$$ Applying linearity of expectation and breaking apart the sum, we have
\begin{align*}
    \bE_t\left[\sum\limits_{t'=t}^T \bI_{t'}(\tilde{A}_{t'}; (A_{t'}, O_{t'+1}))\right] &\leq \bE_t\left[\sum\limits_{t'=t}^T \left(\mc{R}_{t'}(D) - \bE_{t'}\left[\mc{R}_{t'+1}(D)\right]\right) \right] \\
    &= \sum\limits_{t'=t}^T\bE_t\left[\mc{R}_{t'}(D)\right] - \sum\limits_{t'=t}^T \bE_t\left[\bE_{t'}\left[\mc{R}_{t'+1}(D)\right]\right] \\
    &\leq \sum\limits_{t'=t}^T\bE_t\left[\mc{R}_{t'}(D)\right] - \sum\limits_{t'=t}^{T-1} \bE_t\left[\bE_{t'}\left[\mc{R}_{t'+1}(D)\right]\right] \\
    &= \bE_t\left[\mc{R}_t(D)\right] + \sum\limits_{t'=t+1}^T\bE_t\left[\mc{R}_{t'}(D)\right] - \sum\limits_{t'=t}^{T-1} \bE_t\left[\bE_{t'}\left[\mc{R}_{t'+1}(D)\right]\right] \\
    &= \mc{R}_t(D) + \sum\limits_{t'=t+1}^T\bE_t\left[\mc{R}_{t'}(D)\right] - \sum\limits_{t'=t}^{T-1} \bE_t\left[\bE_{t'}\left[\mc{R}_{t'+1}(D)\right]\right].
\end{align*}
We may complete the proof by applying the tower property of expectation and then re-indexing the last summation
\begin{align*}
    \bE_t\left[\sum\limits_{t'=t}^T \bI_{t'}(\tilde{A}_{t'}; (A_{t'}, O_{t'+1}))\right] &\leq \mc{R}_t(D) + \sum\limits_{t'=t+1}^T\bE_t\left[\mc{R}_{t'}(D)\right] - \sum\limits_{t'=t}^{T-1} \bE_t\left[\bE_{t'}\left[\mc{R}_{t'+1}(D)\right]\right] \\
    &= \mc{R}_t(D) + \sum\limits_{t'=t+1}^T\bE_t\left[\mc{R}_{t'}(D)\right] - \sum\limits_{t'=t}^{T-1} \bE_t\left[\mc{R}_{t'+1}(D)\right] \\
    &= \mc{R}_t(D) + \sum\limits_{t'=t+1}^T\bE_t\left[\mc{R}_{t'}(D)\right] - \sum\limits_{t'=t+1}^T \bE_t\left[\mc{R}_{t'}(D)\right] \\
    &= \mc{R}_t(D).
\end{align*}
\end{proof}

With all of these tools in hand, we may now establish an information-theoretic regret bound. For each time period $t \in [T]$, define the information ratio as $$\Gamma_t \triangleq \frac{\bE_t\left[\overline{\rho}(\tilde{A}_t) - \overline{\rho}(A_t)\right]^2}{\bI_t(\tilde{A}_t; (A_t, O_{t+1}))}.$$ 

\begin{theorem}
For any $D > 0$, if $\forall t \in [T]$ $\Gamma_t \leq \overline{\Gamma} < \infty$, then $$\bE\left[\sum\limits_{t=1}^T (\overline{\rho}(A^\star) - \overline{\rho}(A_t))\right] \leq \sqrt{\overline{\Gamma}T\mc{R}_1(D)} + T\sqrt{D}.$$
\label{thm:rdf_regret_bound}
\end{theorem}
\begin{proof}
First, we establish a simple regret decomposition $$\bE\left[\sum\limits_{t=1}^T (\overline{\rho}(A^\star) - \overline{\rho}(A_t))\right] = \bE\left[\sum\limits_{t=1}^T (\overline{\rho}(A^\star) - \overline{\rho}(\tilde{A}_t) + \overline{\rho}(\tilde{A}_t) - \overline{\rho}(A_t))\right] = \bE\left[\sum\limits_{t=1}^T (\overline{\rho}(A^\star) - \overline{\rho}(\tilde{A}_t))\right] + \bE\left[\sum\limits_{t=1}^T (\overline{\rho}(\tilde{A}_t) - \overline{\rho}(A_t))\right],$$ where the first term captures our cumulative performance shortfall by pursuing a learning target $\tilde{A}_t$ in each time period, rather than $A^\star$, while the second term captures our regret with respect to each target. The latter term is also known as the satisficing regret~\citep{russo2022satisficing}. Focusing on the first term, we may apply the tower property of expectation to leverage the fact that each target action $\tilde{A}_t$ achieves the rate-distortion limit and, therefore, has bounded expected distortion:
\begin{align*}
    \bE\left[\sum\limits_{t=1}^T (\overline{\rho}(A^\star) - \overline{\rho}(\tilde{A}_t))\right] &= \bE\left[\sum\limits_{t=1}^T \bE_t\left[\overline{\rho}(A^\star) - \overline{\rho}(\tilde{A}_t)\right]\right] \\
    &= \bE\left[\sum\limits_{t=1}^T \bE_t\left[\big|\overline{\rho}(A^\star) - \overline{\rho}(\tilde{A}_t)\big|\right]\right] \\
    &= \bE\left[\sum\limits_{t=1}^T \bE_t\left[\sqrt{\left(\overline{\rho}(A^\star) - \overline{\rho}(\tilde{A}_t)\right)^2}\right]\right] \\
    &\leq \bE\left[\sum\limits_{t=1}^T \sqrt{\bE_t\left[\left(\overline{\rho}(A^\star) - \overline{\rho}(\tilde{A}_t)\right)^2\right]}\right] \\
    &= \bE\left[\sum\limits_{t=1}^T \sqrt{\bE_t\left[d(\tilde{A}_t, \mc{E})\right]}\right] \\
    &\leq \bE\left[\sum\limits_{t=1}^T \sqrt{D}\right] \\
    &= T\sqrt{D},
\end{align*}
where the first inequality is due to Jensen's inequality. So, in total, we have established that $$\bE\left[\sum\limits_{t=1}^T (\overline{\rho}(A^\star) - \overline{\rho}(A_t))\right] = \bE\left[\sum\limits_{t=1}^T (\overline{\rho}(A^\star) - \overline{\rho}(\tilde{A}_t))\right] + \bE\left[\sum\limits_{t=1}^T (\overline{\rho}(\tilde{A}_t) - \overline{\rho}(A_t))\right] \leq \bE\left[\sum\limits_{t=1}^T (\overline{\rho}(\tilde{A}_t) - \overline{\rho}(A_t))\right] + T\sqrt{D}.$$

The remainder of the proof follows as a standard information-ratio analysis~\citep{russo2016information}, only now with the provision of Lemma \ref{lemma:cum_info_bound}. Namely, we have
\begin{align*}
    \bE\left[\sum\limits_{t=1}^T (\overline{\rho}(\tilde{A}_t) - \overline{\rho}(A_t))\right] &= \bE\left[\sum\limits_{t=1}^T \bE_t\left[\overline{\rho}(\tilde{A}_t) - \overline{\rho}(A_t)\right]\right] \\
    &= \bE\left[\sum\limits_{t=1}^T \sqrt{\Gamma_t \bI_t(\tilde{A}_t; (A_t, O_{t+1})}\right] \\
    &\leq \sqrt{\overline{\Gamma}} \bE\left[\sum\limits_{t=1}^T \sqrt{\bI_t(\tilde{A}_t; (A_t, O_{t+1})}\right] \\
    &\leq \sqrt{\overline{\Gamma}T \bE\left[\sum\limits_{t=1}^T \bI_t(\tilde{A}_t; (A_t, O_{t+1})\right]} \\
    &\leq \sqrt{\overline{\Gamma}T \mc{R}_1(D)},
\end{align*}
where the first inequality follows from our uniform upper bound to the information ratios, the second inequality is the Cauchy-Scwharz inequality, and the final inequality is due to Lemma \ref{lemma:cum_info_bound}. Putting everything together, we have established that $$\bE\left[\sum\limits_{t=1}^T (\overline{\rho}(A^\star) - \overline{\rho}(A_t))\right] \leq \sqrt{\overline{\Gamma}T \mc{R}_1(D)} + T\sqrt{D}.$$
Theorem 1 then follows by Proposition 3 of \citet{russo2016information}, which establishes that $\overline{\Gamma} = \frac{1}{2}|\mc{A}|$ for a multi-armed bandit problem with rewards bounded in the unit interval and a finite action space.
\end{proof}

\end{document}